\PassOptionsToPackage{unicode,colorlinks}{hyperref}
\PassOptionsToPackage{hyphens}{url}
\PassOptionsToPackage{dvipsnames,svgnames,x11names}{xcolor}

\documentclass{article}

\usepackage[nonatbib,preprint]{neurips_2022}

\usepackage[utf8]{inputenc} 
\usepackage[T1]{fontenc}    
\usepackage{url}            
\usepackage{booktabs}       
\usepackage{amsfonts}       
\usepackage{nicefrac}       
\usepackage{microtype}      
\usepackage{xcolor}         

\makeatletter
\def\fps@figure{htbp}
\makeatother

\usepackage{algorithm,algorithmic}

\usepackage{amsmath}
\usepackage{amssymb}
\usepackage{mathtools}
\usepackage{amsthm}

\usepackage{multirow}
\usepackage{csvsimple}
\usepackage{siunitx}

\usepackage{hyperref}
\usepackage[capitalize,noabbrev]{cleveref}
\hypersetup{
    colorlinks = true,
    linkcolor  = Maroon,
    filecolor  = Maroon,
    citecolor  = Navy,
    urlcolor   = Navy,
    breaklinks = true
}

\theoremstyle{plain}
\newtheorem{theorem}{Theorem}[section]

\newtheorem{lemma}[theorem]{Lemma}

\theoremstyle{definition}

\theoremstyle{remark}
\newtheorem{remark}[theorem]{Remark}

\DeclareMathOperator*{\argmin}{arg\,min}

\DeclareMathOperator{\sign}{sign}

\usepackage[style=numeric,backend=biber,sorting=none,url=false]{biblatex}
\title{The Hessian Screening Rule}

\author{%
    Johan Larsson\\
    Department of Statistics\\
    Lund University\\
    \texttt{johan.larsson@stat.lu.se} \\
    \And
    Jonas Wallin\\
    Department of Statistics\\
    Lund University\\
    \texttt{jonas.wallin@stat.lu.se} \\
}

\begin{document}

\maketitle

\begin{abstract}
  Predictor screening rules, which discard predictors before fitting a model,
  have had considerable impact on the speed with which sparse regression
  problems, such as the lasso, can be solved. In this paper we present a new
  screening rule for solving the lasso path: the Hessian Screening Rule. The
  rule uses second-order information from the model to provide both effective
  screening, particularly in the case of high correlation, as well as
  accurate warm starts. The proposed rule outperforms all alternatives we
  study on simulated data sets with both low and high correlation for
  \(\ell_1\)-regularized least-squares (the lasso) and logistic regression.
  It also performs best in general on the real data sets that we examine.
\end{abstract}

\section{Introduction}
\label{sec:introduction}

High-dimensional data, where the number of features (\(p\)) exceeds the number
of observations (\(n\)), poses a challenge for many classical statistical
models. A common remedy for this issue is to regularize the model by penalizing
the regression coefficients such that the solution becomes sparse. A popular
choice of such a penalization is the \(\ell_1\)-norm, which, when the objective
is least-squares, leads to the well-known lasso~\autocite{tibshirani1996}. More
specifically, we will focus on the following convex optimization problem:
\begin{equation}
  \operatorname*{minimize}_{\beta \in \mathbb{R}^p}
  \big\{ f(\beta; X) + \lambda\lVert \beta\rVert_1 \big\},
  \label{eq:primal}
\end{equation}
where \(f(\beta; X)\) is smooth and convex. We let \(\hat\beta\) be the solution
vector for this problem and, abusing notation, equivalently let \(\hat\beta :
\mathbb{R} \mapsto \mathbb{R}^p\) be a function that returns this vector for a
given \(\lambda\). Our focus lies in solving~\eqref{eq:primal}
along a regularization path \(\lambda_1,\lambda_2\,\dots, \lambda_m\) with
\(\lambda_1 \geq \lambda_2 \geq \cdots \geq \lambda_m\). We start the path at
\(\lambda_\text{max}\), which corresponds to the null (all-sparse)
model\footnote{\(\lambda_\text{max}\) is in fact available in closed form---for
  the lasso it is \(\max_j | x_j^T y|\).}, and finish at some fraction of
\(\lambda_\text{max}\) for which the model is either almost saturated (in the
\(p \geq n\) setting), or for which the solution approaches the ordinary
least-squares estimate. The motivation for this focus is that the optimal
\(\lambda\) is typically unknown and must be estimated through model tuning,
such as cross-validation. This involves repeated refitting of the model to new
batches of data, which is computationally demanding.

Fortunately, the introduction of so-called \emph{screening rules} has improved
this situation remarkably. Screening rules use tests that screen and possibly
discard predictors from the model \emph{before} it is fit, which effectively
reduces the dimensions of the problem and leads to improvements in performance
and memory usage. There are, generally speaking, two types of screening rules:
\emph{safe} and \emph{heuristic} rules. Safe rules guarantee that discarded
predictors are inactive at the optimum---heuristic rules do not and may
therefore cause violations: discarding active predictors. The possibility of
violations mean that heuristic methods need to validate the solution through
checks of the Karush--Kuhn--Tucker (KKT) optimality conditions after
optimization has concluded and, whenever there are violations, re-run
optimization, which can be costly particularly because the KKT checks
themselves are expensive. This means that the distinction between safe and
heuristic rules only matters in regards to algorithmic details---all heuristic
methods that we study here use KKT checks to catch these violations, which
means that these methods are in fact also safe.

Screening rules can moreover also be classified as \emph{basic},
\emph{sequential}, or \emph{dynamic}. Basic rules screen predictors based only
on information available from the null model. Sequential rules use information
from the previous step(s) on the regularization path to screen predictors for
the next step. Finally, dynamic rules screen predictors during optimization,
reducing the set of screened predictors repeatedly throughout optimization.

Notable examples of safe rules include the basic SAFE
rule~\autocite{elghaoui2010}, the sphere tests~\autocite{xiang2011}, the
R-region test~\autocite{xiang2012}, Slores~\autocite{wang2014},
Gap~Safe~\cite{fercoq2015,ndiaye2017}, and Dynamic Sasvi~\autocite{yamada2021}.
There is also a group of dual polytope projection rules, most prominently
Enhanced Dual Polytope Projection (EDPP)~\autocite{wang2015}. As noted
by~\textcite{fercoq2015}, however, the sequential version of EDPP relies on
exact knowledge of the optimal solution at the previous step along the path to
be safe in practice, which is only available for \(\lambda_\text{max}\). Among
the heuristic rules, we have the Strong Rule~\autocite{tibshirani2012},
SIS~\cite{fan2008}, and ExSIS~\autocite{ahmed2019}. But the latter two of these
are not sequential rules and solve a potentially reduced form of the problem
in~\eqref{eq:primal}---we will not discuss them further here. In addition to
these two types of rules, there has also recently been attempts to combine safe
and heuristic rules into so-called hybrid rules~\autocite{zeng2021}.

There are various methods for employing these rules in practice. Of particular
interest are so-called \emph{working set} strategies, which use a subset of the
screened set during optimization, iteratively updating the set based on some
criterion. \textcite{tibshirani2012} introduced the first working set strategy,
which we in this paper will refer to simply as the \emph{working set strategy}.
It uses the set of predictors that have ever been active as an initial working
set. After convergence on this set, it checks the KKT optimality conditions on
the set of predictors selected by the strong rule, and then adds predictors
that violate the conditions to the working set. This procedure is then repeated
until there are no violations, at which point the optimality conditions are
checked for the entire set, possibly triggering additional iterations of the
procedure. Blitz~\autocite{johnson2015} and Celer~\cite{massias2018} are two
other methods that use both Gap Safe screening and working sets. Instead of
choosing previously active predictors as a working set, however, both Blitz and
Celer assign priorities to each feature based on how close each feature is of
violating the Gap Safe check and construct the working set based on this
prioritization. In addition to this, Celer uses dual point acceleration
to improve Gap Safe screening and speed up convergence. Both Blitz and Celer
are heuristic methods.

One problem with current screening rules is that they often become
conservative---including large numbers of predictors into the screened
set---when dealing with predictors that are strongly correlated.
\textcite{tibshirani2012}, for instance, demonstrated this to be the case with
the strong rule, which was the motivation behind the working set strategy. (See
\cref{sec:effectiveness-and-violations}
for additional experiments verifying
this). Yet because the computational complexity of the KKT checks in the
working set strategy still depends on the strong rule, the effectiveness of the
rule may nevertheless be hampered in this situation. A possible and---as we
will soon show---powerful solution to this problem is to make use of the
second-order information available from~\eqref{eq:primal}, and in this paper we
present a novel screening rule based on this idea. Methods using second-order
information (the Hessian) are often computationally infeasible for
high-dimensional problems. We utilize two properties of the problem to remedy
this issue: first, we need only to compute the Hessian for the active set,
which is often much smaller than the full set of predictors. Second, we avoid
constructing the Hessian (and it's inverse) from scratch for each \(\lambda\)
along the path, instead updating it sequentially by means of the Schur
complement. The availability of the Hessian also enables us to improve the warm
starts (the initial coefficient estimate at the start of each optimization run)
used when fitting the regularization path, which plays a key role in our
method.

We present our main results in \cref{sec:main-result}, beginning with a
reformulation of the strong rule and working set strategy before we arrive at
the screening rule that represents the main result of this paper. In
\cref{sec:experiments}, we present numerical experiments on simulated and real
data to showcase the effectiveness of the screening rule, demonstrating that
the rule is effective both when \(p \gg n\) and \(n \gg p\), out-performing the
other alternatives that we study. Finally, in \cref{sec:discussion} we wrap up
with a discussion on these results, indicating possible ways in which they may
be extended.

\section{Preliminaries}
\label{sec:preliminaries}

We use lower-case letters to denote scalars and vectors and upper-case letters
for matrices. We use \(\boldsymbol{0}\) and \(\boldsymbol{1}\) to denote
vectors with elements all equal to 0 or 1 respectively, with dimensions
inferred from context. Furthermore, we let \(\sign\) be the standard signum
function with domain \(\{-1,0,1\}\), allowing it to be overloaded for vectors.

Let \(c(\lambda) \coloneqq -\nabla_{\beta} f\big(\hat\beta(\lambda);X\big)\) be
the negative gradient, or so-called \emph{correlation}, and denote
\(\mathcal{A}_{\lambda}= \{i: |c(\lambda)_i|>\lambda \}\) as the \emph{active
  set} at \(\lambda\): the support set of the non-zero regression coefficients
corresponding to \(\hat\beta (\lambda)\). In the interest of brevity, we will
let \(\mathcal{A} \coloneqq \mathcal{A}_{\lambda}\). We will consider \(\beta\)
a solution to~\eqref{eq:primal} if it satisfies the stationary criterion
\begin{equation}
  \label{eq:stationarity}
  \boldsymbol{0} \in \nabla_{\beta} f(\beta;X) + \lambda \partial.
\end{equation}
Here \(\partial\) is the subdifferential of \(\lVert \beta \rVert_1\),
defined as
\[
  \partial_j \in
  \begin{cases}
    \{\sign(\hat\beta_j)\} & \text{if } \hat\beta_j \neq 0, \\
    [-1, 1]                & \text{otherwise.}
  \end{cases}
\]
This means that there must be a \(\tilde\partial \in \partial\) for a given
\(\lambda\) such that
\begin{equation}
  \label{eq:subgrad-solution}
  \nabla_{\beta} f(\beta;X) + \lambda \tilde\partial = \boldsymbol{0}.
\end{equation}

\section{Main Results}
\label{sec:main-result}

In this section we derive the main result of this paper: the Hessian screening
rule. First, however, we now introduce a non-standard perspective on screening
rules. In this approach, we note that~\eqref{eq:stationarity} suggests a simple
and general formulation for a screening rule, namely: we substitute the
gradient vector in the optimality condition of a \(\ell_1\)-regularized problem
with an estimate. More precisely, we discard the \(j\)th predictor for the
problem at a given \(\lambda\) if the magnitude of the \(j\)th component of the
gradient vector estimate is smaller than this \(\lambda\), that is
\begin{equation} \label{eq:general-screening-rule} |\tilde c(\lambda)_j| <
  \lambda. \end{equation} In the following sections, we review the strong rule
and working set method for this problem from this perspective, that is, by
viewing both methods as gradient approximations. We start with the case of the
standard lasso (\(\ell_1\)-regularized least-squares), where we have \(
f(\beta;X) = \frac{1}{2} \lVert X\beta - y\rVert_2^2. \)

\subsection{The Strong Rule}
\label{sec:strong-rule}

The sequential strong rule for \(\ell_1\)-penalized least-squares
regression~\autocite{tibshirani2012}  discards the \(j\)th predictor at
\(\lambda=\lambda_{k+1}\) if
\[
  \big|x_j^T(X\hat\beta(\lambda_{k}) - y)\big| = |c(\lambda_k)_j| < 2
  \lambda_{k+1} - \lambda_{k}.
\]
This is equivalent to checking that
\begin{equation}
  \tilde{c}^S(\lambda_{k+1})=c(\lambda_{k}) +
  \left(\lambda_{k} - \lambda_{k+1} \right) \sign(c(\lambda_{k}))
  \label{eq:strong-gradient}
\end{equation}
satisfies~\eqref{eq:general-screening-rule}. The strong rule gradient
approximation~\eqref{eq:strong-gradient} is also known as the \emph{unit bound},
since it assumes the gradient of the correlation vector to be bounded by one.

\subsection{The Working Set Method}
\label{sec:working-set}

A simple but remarkably effective alternative to direct use of the strong rule
is the working set heuristic~\autocite{tibshirani2012}. It begins by estimating
\(\beta\) at the \((k+1)\)th step using only the coefficients that have been
previously active at any point along the path, i.e. \(\mathcal{A}_{1:k} =
\cup_{i=1}^k\mathcal{A}_i\). The working set method can be viewed as a gradient
estimate in the sense that
\[
  \tilde  c^W(\lambda_{k+1})
  = X^T \left(y-X_{\mathcal{A}_{1:k}} \tilde{\beta}(\lambda_{k+1},\mathcal{A}_{1:k}) \right)\\
  = -\nabla f\big(\tilde\beta(\lambda_{k+1},\mathcal{A}_{1:k});X\big),
\]
where \(\tilde{ \beta}_{}(\lambda,\mathcal{A})= \argmin_{\beta}
\frac12|| y - X_{\mathcal{A}}\beta||^2 + \lambda |\beta|\).

\subsection{The Hessian Screening Rule}
\label{sec:hessian-screening-rule}

We have shown that both the strong screening rule and the working set strategy
can be expressed as estimates of the correlation (negative gradient) for the
next step of the regularization path. As we have discussed previously, however,
basing this estimate on the strong rule can lead to conservative
approximations. Fortunately, it turns out that we can produce a better estimate
by utilizing second-order information.

We start by noting that~\eqref{eq:subgrad-solution}, in the case of the
standard lasso, can be formulated as
\[
  \begin{bmatrix}
    {X_\mathcal{A}^T}X_\mathcal{A}   & X_\mathcal{A}^TX_{\mathcal{A}^c}
    \\
    X_{\mathcal{A}^c}^TX_\mathcal{A} & X_{\mathcal{A}^c}^T X_{\mathcal{A}^c}
  \end{bmatrix}
  \begin{bmatrix}
    \hat\beta_\mathcal{A} \\ 0
  \end{bmatrix}
  +
  \lambda
  \begin{bmatrix}
    \sign(\hat\beta(\lambda)_\mathcal{A}) \\ \partial_{\mathcal{A}^c}
  \end{bmatrix}
  =
  \begin{bmatrix}
    X_\mathcal{A}^Ty \\ X_{\mathcal{A}^c}^T y
  \end{bmatrix},
\]
and consequently that
\begin{equation*}
  \hat\beta(\lambda)_\mathcal{A} = (X^T_\mathcal{A}
  X_\mathcal{A})^{-1}\big(X^T_\mathcal{A} y - \lambda
  \sign{(\hat\beta_\mathcal{A})}\big).
\end{equation*}
Note that, for an interval \([\lambda_l, \lambda_u]\) in which the active set
is unchanged, that is, \(\mathcal{A}_{\lambda} = \mathcal{A}\) for all
\(\lambda \in [\lambda_u,\lambda_k]\), then
\(\hat\beta(\lambda)\) is a continuous linear function in
\(\lambda\)~(\cref{thm:continuity})\footnote{This result is not a new
  discovery~\autocite{efron2004}, but is included here for convenience because
  the
  following results depend on it.}.

\begin{theorem}
  \label{thm:continuity}
  Let \(\hat\beta(\lambda)\) be the solution of~\eqref{eq:primal} where
  \(f(\beta;X)=\frac 12 \lVert X\beta - y \rVert_2^2\).
  Define
  \[
    \hat{\beta}^{\lambda^*}(\lambda)_{A_{\lambda^*}}
    = \hat\beta(\lambda^*)_{\mathcal{A}_{\lambda^*}}-
    \left(\lambda^* - \lambda \right) \left( X_{\mathcal{A}_{\lambda^*}}^T
    X_{\mathcal{A}_{\lambda^*}}\right)^{-1} \sign
    \big(\hat\beta(\lambda^*)_{\mathcal{A}_{\lambda^*}}\big)
  \]
  and \( \hat{\beta}^{\lambda^*}(\lambda)_{\mathcal{A}_{\lambda^*}^c} = 0.\)
  If it for \(\lambda \in [\lambda_0, \lambda^*]\) holds that
  (i) \(\sign\big(\hat{\beta}^{\lambda^*}(\lambda)\big) =
  \sign  \big(\hat\beta(\lambda^*)\big)\) and (ii)
  \(\max |\nabla f( \hat{\beta}^{\lambda^*}
  (\lambda))_{\mathcal{A}_{\lambda^*}}| < \lambda,\)
  then \(\hat\beta(\lambda)= \hat{\beta}^{\lambda^*}(\lambda)\) for \(\lambda
  \in [\lambda_0, \lambda^*]\).
\end{theorem}
See \cref{sec:proofs} for a full proof. Using
\cref{thm:continuity}, we have
the following second-order approximation of \(c(\lambda_{k+1})\):
\begin{equation}
  \label{eq:secondorder}
  \hat c^H(\lambda_{k+1}) =-\nabla f\big(\hat{\beta}^{\lambda_k}(\lambda_{k+1})_{A_{\lambda_k}}\big)= c(\lambda_k) + (\lambda_{k+1} - \lambda_{k})X^T X_{\mathcal{A}_k} (X_{\mathcal{A}_k}^TX_{\mathcal{A}_k})^{-1}\sign \big(\hat\beta (\lambda_k)_{\mathcal{A}_k}\big).
\end{equation}
\begin{remark}
  If no changes in the active set occur in
  \([\lambda_{k+1},\lambda_k]\), \eqref{eq:secondorder} is in fact an exact
  expression for the correlation at the next step,
  that is, \(\hat c^H(\lambda_{k+1})=c(\lambda_{k+1})\).
\end{remark}
One problem with using the gradient estimate in~\eqref{eq:secondorder} is that
it is expensive to compute due to the inner products involving the full design
matrix. To deal with this, we use the following modification, in which we
restrict the computation of these inner products to the set indexed by the
strong rule, assuming that predictors outside this set remain inactive:
\begin{equation*}
  \tilde c^H(\lambda_{k+1})_j \coloneqq
  \begin{cases}
    \lambda_{k+1}\sign{\hat{\beta}(\lambda_k)}_j
      & \mbox{if } j\in \mathcal{A}_{\lambda_k},                                                                 \\
    0 & \mbox{if } |\tilde c^S(\lambda_{k+1})_j| < \lambda_{k+1}{\mbox{ and } j \notin \mathcal{A}_{\lambda_k}}, \\
    \hat c^H(\lambda_{k+1})_j
      & \mbox{else}.
  \end{cases}
\end{equation*}
For high-dimensional problems, this modification leads to large
computational gains and seldom proves inaccurate, given that
the strong rule only rarely causes violations~\autocite{tibshirani2012}.
Lastly, we make one more adjustment to the rule, which is to add
a proportion of the unit bound (used in the strong rule) to the
gradient estimate:
\begin{equation*}
  \check c^{H}(\lambda_{k+1})_j \coloneqq \tilde c^H(\lambda_{k+1})_j +
  \gamma (\lambda_{k+1}-\lambda_k)\sign(c(\lambda_{k})_j),
\end{equation*}
where \(\gamma \in \mathbb{R}_+\). Without this adjustment
there would be no upwards bias on the estimate, which would cause
more violations than would be desirable.
In our experiments, we have used \(\gamma = 0.01\), which
has worked well for most problems we have encountered.
This finally leads us to the \emph{Hessian screening rule}: discard the
\(j\)th predictor at \(\lambda_{k+1}\) if \(|\check c^H(\lambda_{k+1})_j| <
\lambda_{k+1}\).

We make one more modification in our implementation of the Hessian Screening
Rule, which is to use the union of the ever-active predictors and those screened
by the screening rule as our final set of screened predictors. We note that this
is a marginal improvement to the rule, since violations of the rule are already
quite infrequent. But it is included nonetheless, given that it comes at no cost
and occasionally prevents violations.

As an example of how the Hessian Screening Rule performs, we examine the
screening performance of several different strategies. We fit a full
regularization path to a design with \(n=200\), \(p=20\,000\), and pairwise
correlation between predictors of \(\rho\). (See \cref{sec:experiments} and
\cref{sec:effectiveness-and-violations} for more information on the
setup.) We compute the average number of screened predictors across iterations
of the coordinate descent solver. The results are displayed in
\cref{fig:efficiency-simulated-small} and demonstrate that our method
gracefully handles high correlation among predictors, offering a screened set
that is many times smaller than those produced by the other screening
strategies. In \cref{sec:effectiveness-and-violations} we extend these
results to \(\ell_1\)-regularized logistic regression as well and report the
frequency of violations.

\begin{figure}[hbtp]
  \centering
  \includegraphics{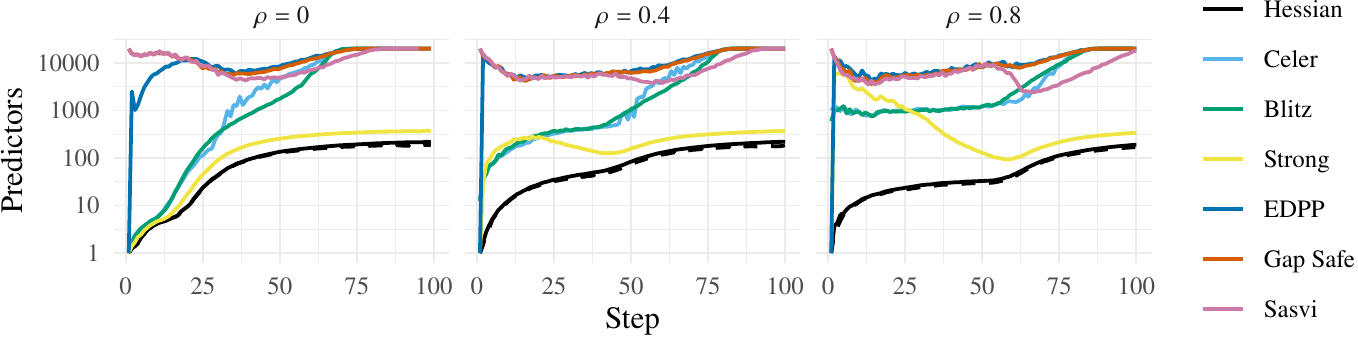}
  \caption{%
    The number of predictors screened (included) for
    when fitting a regularization path of \(\ell_1\)-regularized
    least-squares to a design with varying correlation (\(\rho\)),
    \(n = 200\), and \(p = 20000\). The values are averaged over
    20 repetitions. The minimum number of active
    predictors at each step across iterations is given as a dashed line.
    Note that the y-axis is on a \(\log_{10}\) scale.
    \label{fig:efficiency-simulated-small}
  }
\end{figure}

Recall that the Strong rule bounds its gradient of the correlation vector
estimate at one. For the Hessian rule, there is no such bound. This means that
it is theoretically possible for the Hessian rule to include more predictors
than the Strong rule\footnote{The chance of this happening is tied to the
setting of \(\gamma\).}. In fact, it is even possible to design special cases
where the Hessian rule could be more conservative than the Strong rule. In
practice, however, we have not encountered any situation in which this is the
case.

\subsubsection{Updating the Hessian}
\label{sec:hessian-updates}

A potential drawback to using the Hessian screening rule is the computational
costs of computing the Hessian and its inverse. Let \(\mathcal{A}_k\) be the
active set at step \(k\) on the lasso path. In order to use the Hessian
screening rule we need \(H^{-1}_{k}={(X_{\mathcal{A}_{k}}^T
X_{\mathcal{A}_{k}})}^{-1}\). Computing \({(X_{\mathcal{A}_{k}}^T
    X_{\mathcal{A}_{k}})}^{-1}\) directly, however, has numerical complexity
\(O(|\mathcal{A}_{k}|^3+ |\mathcal{A}_{k}|^2n)\). But if we have stored
\((H^{-1}_{k-1},H_{k-1})\) previously, we can utilize it to compute
\((H_k^{-1},H_k)\) more efficiently via the so-called sweep
operator~\parencite{goodnight1979}. We outline this technique in
\cref{alg:hessian-update}~(\cref{sec:algorithms}). The algorithm has a reduction step
and an augmentation step; in the reduction step, we reduce the Hessian and its
inverse to remove the presence of any predictors that are no longer active. In
the augmentation step, we update the Hessian and its inverse to account for
predictors that have just become active.

The complexity of the steps depends on the size of the
sets \(\mathcal{C}=\mathcal{A}_{k-1}\setminus
\mathcal{A}_{k}\),\(\mathcal{D} = \mathcal{A}_{k}\setminus \mathcal{A}_{k-1}\),
and \(\mathcal{E} = \mathcal{A}_{k}\cap \mathcal{A}_{k-1}\)
The complexity of the reduction step is
\(O(|\mathcal{C}|^{3} + |\mathcal{C}|^{2}|\mathcal{E}| +
|\mathcal{C}| |\mathcal{E}|^2)\) and the complexity of the augmentation step is
\(O(|\mathcal{D}|^2 n  + n|\mathcal{D}||\mathcal{E}| +
|\mathcal{D}|^2|\mathcal{E}| + |\mathcal{D}|^3 )\)
since \(n \geq \max(|\mathcal{E}|,|\mathcal{D}|)\).
An iteration of \cref{alg:hessian-update} therefore has complexity
\(O(|\mathcal{D}|^2 n
+ n|\mathcal{D}||\mathcal{E}| +|\mathcal{C}|^3 + |\mathcal{C}|
|\mathcal{E}|^2)\).

In most applications, the computationally dominant term will be
\(n|\mathcal{D}||\mathcal{E}|\) (since, typically,
\(n>|\mathcal{E}|>\mathcal{D}>\mathcal{C}\)) which could be compared to
evaluating the gradient for \(\beta_{\mathcal{A}_k}\), which is \(n
\left(|\mathcal{D}| + |\mathcal{E}|\right)\) when
\(\beta_{\mathcal{A}^c_k}=0\).
Note that we have so far assumed that the inverse of the Hessian exists, but
this need not be the case. To deal with this issue we precondition the
Hessian. See \cref{sec:nullspace} for details.

\subsubsection{Warm Starts}
\label{sec:warm-starts}

The availability of the Hessian and its inverse
offers a coefficient warm start that is more accurate than the
standard, naive, approach of using the estimate from the previous step.
With the Hessian screening rule, we use the following warm start.
\begin{equation}
  \label{eq:warm-start}
  \hat\beta(\lambda_{k+1})_{\mathcal{A}_k} \coloneqq
  \hat\beta(\lambda_{k})_{\mathcal{A}_k} +
  (\lambda_k - \lambda_{k+1}) H_{\mathcal{A}_k}^{-1}
  \sign\big(\hat\beta(\lambda_k)_{\mathcal{A}_k}\big),
\end{equation}
where \(H_{\mathcal{A}_k}^{-1}\) is the Hessian matrix for the
differentiable part of the objective. Our warm start is equivalent to
the one used in \textcite{park2007}, but is here made much more efficient due
due to the efficient updates of the Hessian and its inverse that we use.
\begin{remark}
  The warm start given by~\eqref{eq:warm-start} is the exact
  solution at \(\lambda_k\) if the active set remains constant
  in \([\lambda_{k+1},\lambda_{k}]\).
\end{remark}
As a first demonstration of the value of this warm start, we look at
two data sets: \emph{YearPredicitionMSD} and \emph{colon-cancer}. We fit
a full regularization path using the setup as outlined in
\cref{sec:experiments}, with or without Hessian warm starts. For
YearPredictionMSD we use the standard lasso, and for colon-cancer
\(\ell_1\)-regularized logistic regression.

The Hessian warm starts offer sizable reductions in the number of passes of the
solver~(\cref{fig:warm-starts}), for many steps requiring only a single pass to
reach convergence. On inspection, this is not a surprising find. There
are no changes in the active set for many of these steps, which means that the
warm start is almost exact---``almost'' due to the use of a preconditioner for
the Hessian (see \cref{sec:nullspace}).

\begin{figure}[hbtp]
  \centering
  \includegraphics{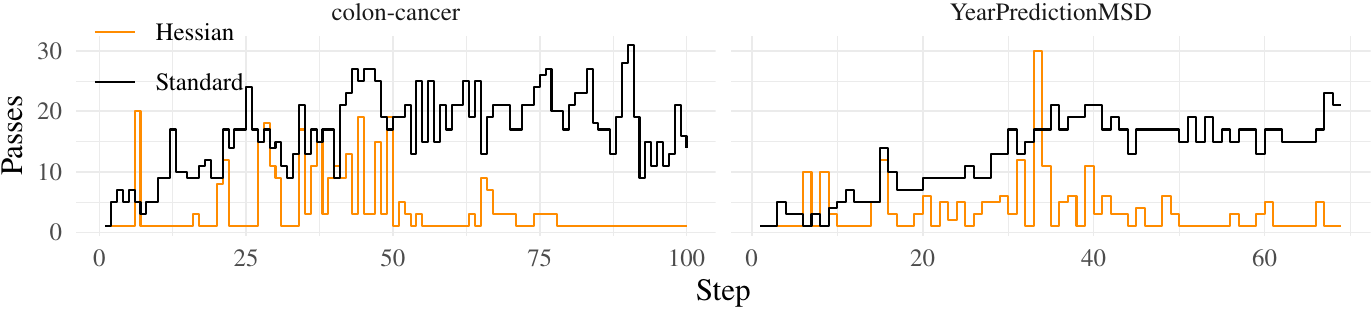}
  \caption{%
    Number of passes of coordinate descent along a full regularization path
    for the \emph{colon-cancer} (\(n = 62\), \(p = 2\,000\)) and
    \emph{YearPredictionMSD} (\(n = 463\,715\), \(p = 90\)) data sets,
    using either Hessian warm starts~\eqref{eq:warm-start} or standard warm
    starts (the solution from the previous step).\label{fig:warm-starts}
  }
\end{figure}

\subsubsection{General Loss Functions}
\label{sec:general-loss-functions}

We now venture beyond the standard lasso and consider loss functions of the
form
\begin{equation}
  \label{eq:gloss}
  f(\beta; X)= \sum_{i=1}^n f_i(x_i^T\beta)
\end{equation}
where \(f_i\) is convex and twice differentiable. This, for
instance, includes logistic, multinomial, and Poisson loss functions. For the
strong rule and working set strategy, this extension does not
make much of a difference. With the Hessian screening rule, however, the
situation is different.

To see this, we start by noting that our method involving the Hessian
is really a quadratic Taylor approximation of~\eqref{eq:primal} around a
specific point \(\beta_0\).
For loss functions of the
type~\eqref{eq:gloss}, this approximation
is equal to
\[
  \begin{aligned}
    Q(\beta,\beta_0) & = f(\beta_0; X) + \sum_{i=1}^n
    \Bigg(x^T_if'_i(x_i^T\beta_0)(\beta -\beta_0) + \frac{1}{2}(\beta -\beta_0)^T x_i^T f''_i(x_i^T\beta_0) x_i (\beta -\beta_0) \Bigg) \\
                     & = \frac{1}{2}\left(\tilde y(x_i^T\beta_0) -
      X\beta \right)^T  D\left(w(\beta_0) \right)
    \left( \tilde y(x_i^T\beta_0) - X\beta \right) +                                                                                    C(\beta_0),
  \end{aligned}
\]
where \(D(w(\beta_0))\) is a diagonal matrix with diagonal entries
\(w(\beta_0)\) where \(w(\beta_0)_i = f''(x ^T_i\beta_0)\) and
\( \tilde y(z)_i =f'_i(z) \big / f''_i(z) - x_i^T\beta_0\), whilst
\(C(\beta_0)\) is a constant with respect to \(\beta\).

Suppose that we are on the lasso path at \(\lambda_k\) and want to
approximate \(c(\lambda_{k+1})\). In this case, we simply replace
\(f(\beta;X)\) in~\eqref{eq:primal} with \(Q(\beta,\hat \beta(\lambda_{k}))\),
which leads to the following gradient approximation:
\[
  c^H(\lambda_{k+1}) =
  c(\lambda_k) + \left( \lambda_{k+1} -
  \lambda_{k}\right)X^T
  D(w) X_{\mathcal{A}_k}(X_{\mathcal{A}_k}^TD(w)X_{\mathcal{A}_k})^{-1}
  \sign \big(\hat\beta(\lambda_{k})_{\mathcal{A}_k}\big),
\]
where \(w =w\big(\hat{\beta}(\lambda_k) \big)\). Unfortunately, we cannot use
\cref{alg:hessian-update} to update
\(X_{\mathcal{A}_k}^T D(w) X_{\mathcal{A}_k}\). This means that we are forced
to either update the Hessian directly at each step, which can be
computationally demanding when \(|\mathcal{A}_k|\) is large and inefficient
when \(X\) is very sparse, or to approximate \(D(w)\) with an upper bound. In
logistic regression, for instance, we can use \nicefrac{1}{4} as such a bound,
which also means that we once again can use
\cref{alg:hessian-update}.

In our experiments, we have employed the following heuristic to
decide whether to use an upper bound or compute the full Hessian
in these cases: we use full updates at each
step if \(\operatorname{sparsity}(X) n / \max\{n,p\} < 10^{-3}\) and
the upper bound otherwise.

\subsubsection{Reducing the Impact of KKT Checks}

The Hessian Screening Rule is heuristic, which means there may be violations.
This necessitates that we verify the KKT conditions after having reached
convergence for the screened set of predictors, and add predictors back into the
working set for which these checks fail. When the screened set is small relative
to \(p\), the cost of optimization is often in large part consumed by these
checks. Running these checks for the full set of predictors always needs to be
done once, but if there are violations during this step, then we need repeat
this check, which is best avoided. Here we describe two methods to
tackle this issue.

We employ a procedure equivalent to the one used in \textcite{tibshirani2012} for
the working set strategy: we first check the KKT conditions for the set of
predictors singled out by the strong rule and then, if there are no violations
in that set, check the full set of predictors for violations. This works well
because the strong rule is conservative---violations are rare---which means that
we seldom need to run the KKT checks for the entire set more than once.

If we, in spite of the augmentation of the rule, run into violations when
checking the full set of predictors, that is, when the strong rule fails to
capture the active set, then we can still avoid repeating the full KKT check by
relying on Gap Safe screening: after having run the KKT checks and have failed
to converge, we screen the set of predictors using the Gap Safe rule. Because
this is a safe rule, we can be sure that the predictors we discard will be
inactive, which means that we will not need to include them in our upcoming KKT
checks. Because Gap Safe screening and the KKT checks rely on exactly the same
quantity---the correlation vector--we can do so at marginal extra cost. To see
how this works, we now briefly introduce Gap Safe screening. For details,
please see \textcite{fercoq2015}.

For the ordinary lasso (\(\ell_1\)-regularized least squares), the primal
\eqref{eq:primal} is
\(
P(\beta) =
\frac 1 2 \lVert y - X\beta\rVert_2^2 + \lambda \lVert \beta \rVert_1
\)
and the corresponding dual is
\begin{equation}
  \label{eq:dual}
  D(\theta) =
  \frac{1}{2} \lVert y \rVert_2^2 -
  \frac{\lambda^2}{2} \Big\lVert \theta -
  \frac{y}{\lambda}\Big\rVert_2^2
\end{equation}
subject to \(\lVert X^T \theta \rVert_\infty \leq 1\).
The duality gap is then \(G(\beta, \theta) =
P(\beta) - D(\theta)\)
and the relation between the primal and dual problems is given by \(y = \lambda
\hat\theta + X\hat\beta\), where \(\hat\theta\) is the maximizer
to the dual problem \eqref{eq:dual}.
In order to use Gap Safe screening, we need a feasible dual point, which can
be obtained via dual point scaling, taking
\(
\theta =
(y - X\beta) \big/ \max\big(\lambda, \lVert X^T(y - X\beta)\rVert_\infty\big).
\)
The Gap Safe screening rule then discards the \(j\)th feature if
\(
\lvert x_j^T\theta \rvert < 1 - \lVert x_j \rVert_2
\sqrt{2G(\beta,\theta)/\lambda^2}.
\)
Since we have computed \(X^T(y - X\beta)\) as part of the KKT
checks, we can perform Gap Safe screening at an additional (and marginal)
cost amounting to \(O(n) + O(p)\).

Since this augmentation benefits the working set strategy too, we adopt it
in our implementation of this method as well. To avoid
ambiguity, we call this version working+. Note that this makes the working
set strategy quite similar to Blitz. In
\cref{sec:ablation} we show the benefit of adding this type of screening.

\subsubsection{Final Algorithm}

The Hessian screening method is presented in full in
\cref{alg:hessian-screening}~(\cref{sec:algorithms}).

\begin{lemma}
  Let \(\beta \in \mathbb{R}^{p \times m}\) be the output of
  \cref{alg:hessian-screening} for a path of length \(m\) and convergence
  threshold \(\varepsilon > 0\). For each step \(k\) along the path and
  corresponding solution \(\beta^{(k)} \in \mathbb{R}^p\), there is a
  dual-feasible point \(\theta^{(k)}\) such that \(G(\beta^{(k)}, \theta^{(k)})
  < \zeta \varepsilon.\)
\end{lemma}
\begin{proof}
  First note that Gap safe screening~\autocite[Theorem 6]{ndiaye2017} ensures
  that \(\mathcal{G} \supseteq \mathcal{A}_k\). Next, note that the algorithm
  guarantees that the working set, \(\mathcal{W}\), grows with each iteration
  until \(|x_j^Tr| < \lambda_k\) for all \(j \in \mathcal{G} \setminus
  \mathcal{W}\), at which point
  \[
    \max\big(\lambda_k, \lVert X^T_\mathcal{W}(y - X_\mathcal{W}\beta_\mathcal{W}^{(k)})\rVert_\infty\big)
    =
    \max\big(\lambda_k, \lVert X^T_\mathcal{G}(y - X_\mathcal{G}\beta_\mathcal{G}^{(k)})\rVert_\infty\big).
  \]
  At this iteration, convergence at line 2, for the subproblem \((X_\mathcal{W}, y)\),
  guarantees convergence for the full problem, \((X, y)\), since
  \[
    \theta^{(k)} = \frac{y - X_\mathcal{W}\beta_\mathcal{W}^{(k)} }{\max\big(\lambda_k, \lVert X^T_\mathcal{W}(y - X_\mathcal{W} \beta_\mathcal{W}^{(k)})\rVert_\infty \big)}
  \]
  is dual-feasible for the full problem.
\end{proof}

\subsubsection{Extensions}
\label{sec:extensions}

\paragraph{Approximate Homotopy}
\label{sec:approximate-homotopy}

In addition to improved screening and warm starts, the Hessian also allows us to
construct the regularization path adaptively via approximate
homotopy~\autocite{mairal2012}. In brief, the Hessian screening rule
allows us to choose the next \(\lambda\) along the path adaptively,
in effect distributing the grid of \(\lambda\)s to better approach the exact
(homotopy) solution for the lasso, avoiding the otherwise heuristic choice,
which can be inappropriate for some data sets.

\paragraph{Elastic Net}
\label{sec:elastic-net}

Our method can be extended to the elastic net~\parencite{zou2005}, which
corresponds to adding a quadratic penalty \(\phi \lVert \beta \rVert_2^2/2\) to
\eqref{eq:primal}. The Hessian now takes the form \(X_\mathcal{A}^T
X_\mathcal{A} + \phi I\). Loosely speaking, the addition of this term makes the
problem ``more`` quadratic, which in turn improves both the accuracy and
stability of the screening and warm starts we use in our method. As far as we
know, however, there is unfortunately no way to update the inverse of the
Hessian efficiently in the case of the elastic net. More research in this area
would be welcome.

\section{Experiments}
\label{sec:experiments}

Throughout the following experiments, we scale and center predictors with the
mean and uncorrected sample standard deviation respectively. For the lasso, we
also center the response vector, \(y\), with the mean.

To construct the regularization path, we adopt the default settings from
\texttt{glmnet}: we use a log-spaced path of 100 \(\lambda\) values from
\(\lambda_\text{max}\) to \(\xi\lambda_\text{max}\), where \(\xi = 10^{-2}\) if
\(p > n\) and \(10^{-4}\) otherwise. We stop the path whenever the deviance
ratio, \(1 - \text{dev}/\text{dev}_\text{null}\), reaches 0.999
or the fractional decrease in deviance is less than \(10^{-5}\).
Finally, we also stop the path whenever the number of coefficients ever to be
active predictors exceeds \(p\).

We compare our method against working+ (the modified version of the working set
strategy from \textcite{tibshirani2012}), Celer~\autocite{massias2018}, and
Blitz~\autocite{johnson2015}. We initially also ran our comparisons against
EDPP~\autocite{wang2015}, the Gap Safe rule~\autocite{fercoq2015}, and Dynamic
Sasvi~\autocite{yamada2021} too, yet these methods performed so poorly that we
omit the results in the main part of this work. The interested reader may
nevertheless consult \cref{sec:experiments-simulateddata-extra}
where results from simulated data has been
included for these methods too.

We use cyclical coordinate descent with shuffling and consider the model to
converge when the duality gap \(G(\beta,\theta) \leq \varepsilon \zeta\), where
we take \(\zeta\) to be \(\lVert y\rVert_2^2\) when fitting the ordinary lasso,
and \(n\log{2}\) when fitting \(\ell_1\)-regularized logistic regression.
Unless specified, we let \(\varepsilon = 10^{-4}\). These settings are standard
settings and, for instance, resemble the defaults used in Celer. For all of the
experiments, we employ the line search algorithm used in Blitz\footnote{Without
  the line search, all of the tested methods ran into convergence issues,
  particularly for the high-correlation setting and logistic regression.}.

The code used in these experiments was, for every method, programmed in C++
using the Armadillo library~\autocite{eddelbuettel2014,sanderson2016} and
organized as an R package via Rcpp~\autocite{eddelbuettel2011}. We used the
renv package~\autocite{ushey2021} to maintain dependencies. The source code,
including a Singularity~\autocite{kurtzer2017} container and its recipe for
reproducing the results, are available at
\url{https://github.com/jolars/HessianScreening}. Additional details of
the computational setup are provided in \cref{sec:computational-setup-details}.

\subsection{Simulated Data}
\label{sec:experiments-simulated}

Let \(X \in \mathbb{R}^{n \times p}\), \(\beta \in \mathbb{R}^{p}\), and \(y
\in \mathbb{R}^n\) be the predictor matrix, coefficient vector, and response
vector respectively. We draw the rows of the predictor matrix independently
and identically distributed from \(\mathcal{N}(0, \Sigma)\) and generate the
response from \(\mathcal{N}(X\beta, \sigma^2I)\) with \(\sigma^2 =
\beta^T\Sigma\beta/\text{SNR}\), where SNR is the signal-to-noise ratio.
We set \(s\) coefficients, equally spaced throughout
the coefficient vector, to 1 and the rest to zero.

In our simulations, we consider two scenarios: a low-dimensional scenario and a
high-dimensional scenario. In the former, we set \(n = 10\,000\), \(p = 100\),
\(s = 5\), and the SNR to 1. In the high-dimensional scenario, we take \(n =
400\), \(p = 40\,000\), \(s = 20\), and set the SNR to 2. These SNR values are
inspired by the discussion in \textcite{hastie2020} and intend to cover the
middle-ground in terms of signal strength. We run our simulations for 20
iterations.

From \cref{fig:performance-simulated}, it is clear that the Hessian screening
rule performs best, taking the least time in every setting examined. The
difference is largest for the high-correlation context in the low-dimensional
setting and otherwise roughly the same across levels of correlation.

The differences between the other methods are on average small, with the
working+ strategy performing slightly better in the \(p > n\) scenario. Celer
and Blitz perform largely on par with one another, although Celer sees an
improvement in a few of the experiments, for instance in logistic regression
when \(p > n\).

\begin{figure}[htb]
  \centering
  \includegraphics{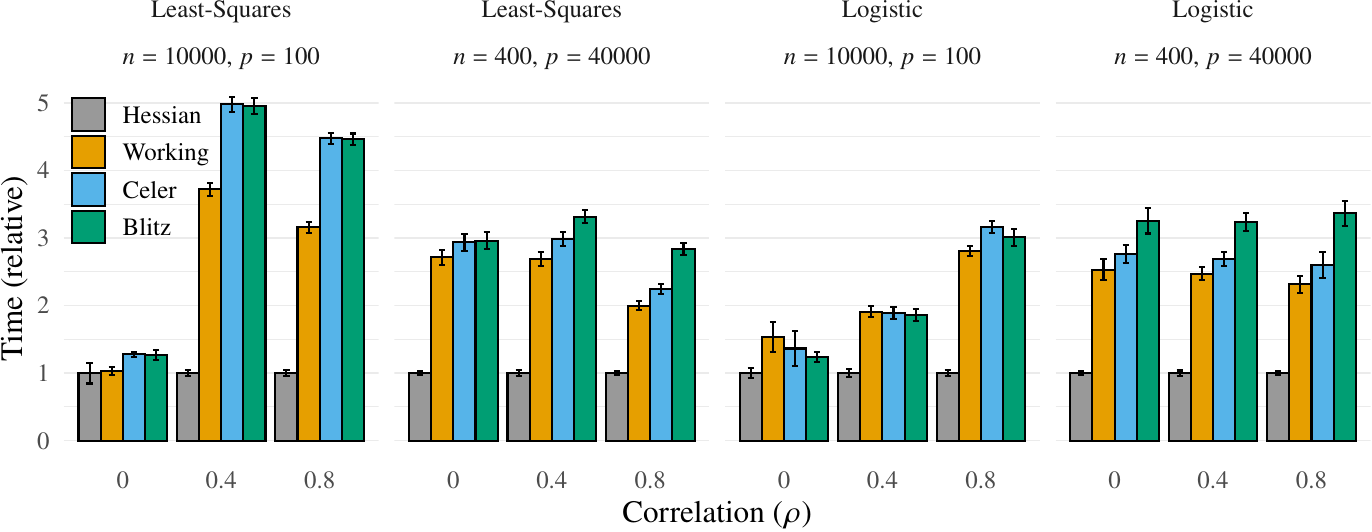}
  \caption{
    Time to fit a full regularization path for \(\ell_1\)-regularized
    least-squares and logistic regression to a design with \(n\)
    observations, \(p\) predictors, and pairwise correlation between
    predictors of \(\rho\). Time is relative to the minimal mean time in
    each group. The error bars represent ordinary 95\% confidence
    intervals around the mean. \label{fig:performance-simulated}
  }
\end{figure}

\subsection{Real Data}
\label{sec:experiments-real}

In this section, we conduct experiments on real data sets. We run 20 iterations
for the smaller data sets studied and three for the larger ones. For information
on the sources of these data sets, please see \cref{sec:real-datasets}.
For more detailed results of these experiments, please see
\cref{sec:experiments-realdata-detailed}.

Starting with the case of \(\ell_1\)-regularized least-squares regression, we
observe that the Hessian screening rule performs best for all five data sets
tested here~(\cref{tab:performance-realdata}), in all but one instance taking
less than half the time compared to the runner-up, which in each case is the
working+ strategy. The difference is particularly large for the
YearPredictionMSD and e2006-tfidf data sets.

In the case of \(\ell_1\)-regularized logistic regression, the Hessian method
again performs best for most of the examined data sets, for instance
completing the regularization path for the madelon data set around five times
faster than the working+ strategy. The exception is
the arcene data set, for which the working+ strategy performs best
out of the four methods.

\begin{table}[hbtp]
  \caption{
    Average time to fit a full regularization path of \(\ell_1\)-regularized
    least-squares and logistic regression to real data sets. Density
    represents the fraction of non-zero entries in \(X\). Density and
    time values are rounded to two and three significant figures
    respectively.\label{tab:performance-realdata}}
  \addtolength{\tabcolsep}{-2pt}
  \csvreader[
    tabular={
        l
        S[table-format=6.0,round-mode=off]
        S[table-format=7.0,round-mode=off]
        S[table-format=1.1e-1,scientific-notation=true,round-precision=2]
        l
        S[table-format=4.4]
        S[table-format=4.3]
        S[table-format=4.3]
        S[table-format=4.3]
      },
    before reading=\scriptsize\centering\sisetup{round-mode=figures,
      round-precision=3},
    table head=\toprule & & & & & \multicolumn{4}{c}{Time (s)} \\
    \cmidrule(lr){6-9} Data Set & {\(n\)} & {\(p\)} & {Density} & {Loss} & {Hessian} &
    {Working} & {Blitz} & {Celer} \\\midrule,
    table foot=\bottomrule
  ]%
  {tables/realdata-timings.csv}%
  {dataset=\dataset, n=\n, p=\p, density=\density, model=\loss, Hessian=\hessian,
    Working=\working, Blitz=\blitz,
    Celer=\celer}%
  {\dataset & \n & \p & \density & \loss & \hessian & \working & \blitz & \celer}
\end{table}

We have provided additional results related to the effectiveness of our method in
\cref{sec:additional-results}.

\section{Discussion}\label{sec:discussion}

In this paper, we have presented the Hessian Screening Rule: a new heuristic
predictor screening rule for \(\ell_1\)-regularized generalized linear models.
We have shown that our screening rule offers large performance improvements
over competing methods, both in simulated experiments but also in the majority
of the real data sets that we study here. The improved performance of the rule
appears to come not only from improved effectiveness in screening, particularly
in the high-correlation setting, but also from the much-improved warm starts,
which enables our method to dominate in the \(n \gg p\) setting. Note that
although we have focused on \(\ell_1\)-regularized least-squares and logistic
regression here, our rule is applicable to any composite objective for which
the differentiable part is twice-differentiable.

One limitation of our method is that it consumes more memory than its
competitors owing to the storage of the Hessian and its inverse. This
cost may become prohibitive for cases when \(\min\{n,p\}\) is
large. In these situations the next-best choice may instead
be the working set strategy.
Note also that we, in this paper, focus entirely on the lasso \emph{path}. The
Hessian Screening Rule is a sequential rule and may therefore not prove optimal
when solving for a single \(\lambda\), in which case a dynamic strategy such as
Celer and Blitz likely performs better.

With respect to the relative performance of the working set strategy, Celer, and
Blitz, we note that our results deviate somewhat from previous
comparisons~\autocite{massias2018,johnson2015}. We speculate that these differences
might arise from the fact that we have used equivalent implementations for all
of the methods and from the modification that we have used for the working set
strategy.

Many avenues remain to be explored in the context of Hessian-based screening
rules and algorithms, such as developing more efficient methods for updating of
the Hessian matrix for non-least-squares objectives, such as logistic
regression and using second-order information to further improve the
optimization method used. Other interesting directions also include adapting
the rules to more complicated regularization problems, such as the fused
lasso~\autocite{tibshirani2005}, SLOPE~\autocite{bogdan2015},
SCAD~\parencite{fan2001}, and MCP~\parencite{zhang2010}. Although the latter
two of these are non-convex problems, they are locally convex for intervals of
the regularization path~\parencite{breheny2011}, which enables the use of our
method.

Finally, we do not expect there to be any negative societal consequences of our
work given that it is aimed solely at improving the performance of an
optimization method.

\begin{ack}
  We would like to thank Małgorzata Bogdan for valuable and encouraging
  comments. This work was funded by the Swedish Research Council through grant
  agreement no. 2020-05081 and no. 2018-01726. The computations were enabled by
  resources provided by LUNARC. The results shown here are in whole or part
  based upon data generated by the TCGA Research Network:
  \url{https://www.cancer.gov/tcga}.
\end{ack}

\printbibliography

\appendix

\section{Proofs}\label{sec:proofs}

\subsection{Proof of Theorem 1}

It suffices to verify that the KKT conditions hold for
\(\hat{\beta}^{\lambda^*}(\lambda)\), i.e. that \(\boldsymbol{0}\) is in
the subdifferential. By (ii) it follows that the indices
\(\mathcal{A}^c_{\lambda^*}\) in the subdifferential contain zero. That
leaves us only to show that \(\nabla
f\big(\hat{\beta}^{\lambda^*}(\lambda);X\big)_{\mathcal{A}_{\lambda^*}}= \lambda
\sign\big(\hat{\beta}^{\lambda^*}(\lambda)\big)_{\mathcal{A}_{\lambda^*}}\).
\begin{align*}
   & \nabla f\big(\hat{\beta}^{\lambda^*}(\lambda);X\big)_{\mathcal{A}_{\lambda^*}} \\
   & = X^T_{\mathcal{A}_{\lambda^*}} \big(y- X_{\mathcal{A}_{\lambda^*}}
  \hat{\beta}^{\lambda^*} (\lambda)_{\mathcal{A}_{\lambda^*}} \big)                 \\
   & = X^T_{\mathcal{A}_{\lambda^*}} \Bigg(y- X_{\mathcal{A}_{\lambda^*}}
  \beta(\lambda^*)_{\mathcal{A}_{\lambda^*}}-
  \big(\lambda^* - \lambda \big){X_{\mathcal{A}_{\lambda^*}}
    \big( X_{\mathcal{A}_{\lambda^*}}^TX_{\mathcal{A}_{\lambda^*}}
    \big)^{-1} \sign \beta(\lambda^*)_{\mathcal{A}_{\lambda^*}}}
  \Bigg)                                                                            \\
   & = \nabla f\big(\hat{\beta}^{\lambda^*}
  (\lambda^*)\big)_{\mathcal{A}_{\lambda^*}} - \left(\lambda^* - \lambda
  \right) \sign \hat{\beta}(\lambda^*)_{\mathcal{A}_{\lambda^*}}                    \\
   & = \lambda \sign \hat{\beta}(\lambda^*)_{\mathcal{A}_{\lambda^*}},
\end{align*}
which by (i) equals
\(\lambda \sign(\hat{\beta}^{\lambda^*}(\lambda))_{\mathcal{A}_{\lambda^*}}\).

\section{Algorithms}
\label{sec:algorithms}

\begin{algorithm}[hbtp]
  \caption{This algorithm provides computationally efficient updates for
    the inverse of the Hessian. Note the slight abuse of notation here in
    that \(\mathcal{E}\) is used both for \(X\) and \(Q\). It is implicitly
    understood that \(Q_{\mathcal{E}\mathcal{E}}\) is the sub-matrix of
    \(Q\)
    that corresponds to the columns \(\mathcal{E}\) of
    \(X\).}
  \label{alg:hessian-update}
  \begin{algorithmic}
    \REQUIRE \(X, H = X_\mathcal{A}^TX_\mathcal{A}, Q \coloneqq H^{-1},
    \mathcal{A}, \mathcal{B}\)
    \STATE \(\mathcal{C} \coloneqq \mathcal{A} \setminus \mathcal{B}\)
    \STATE \(\mathcal{D} \coloneqq \mathcal{B} \setminus \mathcal{A}\)
    \IF{\(\mathcal{C} \neq \varnothing\)}
    \STATE \(\mathcal{E} \coloneqq \mathcal{A} \cap \mathcal{B}\)
    \STATE \(Q \coloneqq Q_{\mathcal{E}\mathcal{E}} -
    Q_{\mathcal{E}\mathcal{E}^c}Q^{-1}_{\mathcal{E}^c\mathcal{E}^c}
    Q_{\mathcal{E}\mathcal{E}^c}^T\)
    \STATE \(\mathcal{A} \coloneqq \mathcal{E}\)
    \ENDIF

    \IF{\(\mathcal{D} \neq \varnothing\)}
    \STATE \(S \coloneqq X_\mathcal{D}^TX_\mathcal{D} -
    X_\mathcal{D}^TX_\mathcal{A}QX^T_\mathcal{A} X_\mathcal{D}\)
    \STATE \(Q \coloneqq
    \begin{bmatrix}
      Q + Q
      X_\mathcal{A}^TX_\mathcal{D}S^{-1}X_\mathcal{D}^TX_\mathcal{A} Q &
      - Q
      X_\mathcal{A}^TX_\mathcal{D}S^{-1}
      \\
      -S^{-1}X_\mathcal{D}^TX_\mathcal{A}Q
                                                                       &
      S^{-1}
    \end{bmatrix}\)
    \ENDIF
    \STATE Return \(H^*\)
  \end{algorithmic}
\end{algorithm}

\begin{algorithm}[hbtp]
  \caption{The Hessian screening method for the ordinary least-squares lasso}
  \label{alg:hessian-screening}
  \begin{algorithmic}[1]
    \REQUIRE \(X \in \mathbb{R}^{n \times p}\), \(y \in \mathbb{R}^n\), \(\lambda \in \{\mathbb{R}^m_+ : \lambda_1 = \lambda_\text{max}, \lambda_1 > \lambda_2 > \cdots > \lambda_m\}\), \(\varepsilon > 0\)
    \ENSURE \(k \gets 1\), \(\beta^{(0)} \gets 0\), \(\zeta \gets \lVert y \rVert_2^2\), \(\mathcal{W} \gets \varnothing\), \(\mathcal{A} \gets \varnothing\), \(\mathcal{S} \gets \varnothing\), \(\mathcal{G} \gets \{1,2,\dots,p\} \)
    \WHILE{\(k \leq m\)}
    \STATE \(\beta_\mathcal{W}^{(k)} \gets \big\{\beta \in \mathbb{R}^{|\mathcal{W}|} : G\big(\beta, (y - X_\mathcal{W}\beta) / \max(\lambda_k, \lVert X^T_\mathcal{W}(y - X_\mathcal{W}\beta)\rVert_\infty )\big) < \zeta \varepsilon\big\}\)
    \STATE \(\beta^{(k)}_{\mathcal{W}^\mathsf{C}} \gets 0\)
    \STATE \(\mathcal{A} \gets \{j : \beta_j \neq 0\}\)
    \STATE \(r \gets y - X_\mathcal{W}\beta^{(k)}_\mathcal{W}\)
    \STATE \(\mathcal{V} \gets \{j \in \mathcal{S} \setminus \mathcal{W} : |x_j^T r| \geq \lambda_k \}\) \COMMENT{Check for violations in Strong set}
    \IF{\(\mathcal{V} = \varnothing\)}
    \STATE \(\theta \gets r / \max\big(\lambda_k, \lVert X_\mathcal{G}^Tr \rVert_\infty \big)\) \COMMENT{Compute dual-feasible point}%
    \IF{\(G(\beta^{(k)}, \theta) < \varepsilon \zeta\)}
    \STATE Update \(H\) and \(H^{-1}\) via \cref{alg:hessian-update}
    \STATE \(\mathcal{W} \gets \{j : |\tilde{c}^H(\lambda_{k + 1})| < \lambda_{k + 1}\} \cup \mathcal{A}\) \COMMENT{Hessian rule screening}%
    \STATE \(\mathcal{S} \gets \{j : |\tilde{c}^S(\lambda_{k + 1})| < \lambda_{k + 1}\}\) \COMMENT{Strong rule screening}%
    \STATE Initialize \(\beta^{(k + 1)}_\mathcal{A}\) using \eqref{eq:warm-start} \COMMENT{Hessian warm start}
    \STATE \(\mathcal{G} \gets \{1, 2, \dots, p\}\) \COMMENT{Reset Gap-Safe set}
    \STATE \(k \gets k + 1\) \COMMENT{Move to next step on path}
    \ELSE
    \STATE \(\mathcal{G} \gets \left\{j \in\mathcal{G} : |x_j^T\theta| \geq 1 - \lVert x_j \rVert_2 \sqrt{2G(\beta^{(k)}, \theta)/\lambda^2_k}\right\}\) \COMMENT{Gap-Safe screening}
    \STATE \(\mathcal{V} \gets \{j \in \mathcal{G} \setminus \big(\mathcal{S} \cup \mathcal{W}\big) : |x_j^T r| \geq \lambda_k \}\) \COMMENT{Check for violations in Gap-Safe set}
    \STATE \(\mathcal{W} \gets \mathcal{W} \cap \mathcal{G}\)
    \STATE \(\mathcal{S} \gets \mathcal{S} \cap \mathcal{G}\)
    \ENDIF
    \ENDIF
    \STATE \(\mathcal{W} \gets \mathcal{W} \cup \mathcal{V}\) \COMMENT{Augment working set with violating predictors}
    \ENDWHILE
    \RETURN \(\beta\)
  \end{algorithmic}
\end{algorithm}

\section{Singular or Ill-Conditioned Hessians}\label{sec:nullspace}

In this section, we discuss situations in which the
Hessian is singular or ill-conditioned and propose remedies for these
situations.

Inversion of the Hessian demands that the null space corresponding to the
active predictors \({\mathcal{A}_\lambda}\) contains only the zero vector,
which typically holds when the columns of \(X\) are in general position,
such as in the case of data simulated from continuous distributions. It is
not, however, generally the case with discrete-valued data, particularly not in
when \(p \gg n\). In \cref{lem:nullspace}, we formalize this point.

\begin{lemma}
  \label{lem:nullspace}
  Suppose that we have \(e \in \mathbb{R}^p\) such that \(Xe=0\). Let
  \(\hat{\beta}(\lambda)\) be the solution to the primal problem
  \eqref{eq:primal} and
  \(\mathcal{E}=\{i: e_i \neq 0\}\); then
  \(|\hat{\beta}(\lambda)_\mathcal{E}|>0\) only if there exists a \(z \in
  \mathbb{R}^p\) where \(z_{\mathcal{E}} \in \{-1,1\}^{|\mathcal{E}|}\)
  such that \({z^T}e=0\).
\end{lemma}
\begin{proof}
  \(\sum_{j\in	\mathcal{E}} x_{j} e_{j}=0\) by assumption. Then, since
  \(\hat{\beta}(\lambda)\) is the solution to the primal problem, it follows
  that \(x_j^T\nabla f(X\beta)= \sign(\beta_j) \lambda\) for all
  \(j \in \mathcal{E}\). Hence
  \[
    \sum_{j\in \mathcal{E}} x_{j}^T\nabla f(X\beta) e_{j}
    = \sum_{j\in \mathcal{E}} \sign(\beta_j) \lambda e_{j}
    = \lambda\sum_{j\in  \mathcal{E}} \sign(\beta_j) e_{j}
    = 0
  \]
  and \(z_{\mathcal{E}^C}=0\),
  \(z_{\mathcal{E}} = \sign(\beta_{\mathcal{E}})\).
\end{proof}
In our opinion, the most salient feature of this result is that if all
predictors in \(\mathcal{E}\) except \(i\) are known to be active, then
predictor
\(i\) is active iff \(e_i=\sum_{j\in \mathcal{E}\setminus i} \pm e_j \). If the
columns of \(X\) are independent and normally distributed, this cannot occur
and hence one will never see a null space in \(X_{\mathcal{A}}\). Yet if
\(X_{ij}\in \{0,1\}\), one should expect the null space to be non-empty
frequently. A simple instance of this occurs when the columns of \(X\) are
duplicates, in which case \(|e|=2\).

Duplicated predictors are fortunately easy to handle since they enter the
model simultaneously. And we have, in our program, implemented measures that
deal efficiently with this issue by dropping them from the solution
after fitting and adjust \(\hat\beta\) accordingly.

Dealing with the presence of rank-deficiencies due to the existence
of linear combinations among the predictors is more challenging. In
the work for this paper, we developed a strategy to deal with this issue
directly by identifying such linear combinations through spectral
decompositions. During our experiments, however, we discovered that this method
often runs into numerical issues that require other modifications that
invalidate its potential. We have therefore opted for a different
strategy.

To deal with singularities and ill-conditioned Hessian matrices, we instead
use preconditioning. At step \(k\),
we form the spectral decomposition \[H_{\mathcal{A}_k} = Q \Lambda Q^T.\]
Then, if \(\min_i\big(\operatorname{diag}(\Lambda)\big) < \alpha\), we add a
factor \(\alpha\) to the diagonal of \(H_{\mathcal{A}_k}\). Then we
substitute \[\hat H_{\mathcal{A}_k}^{-1} = Q^T ( I \alpha + \Lambda)^{-1} Q\]
for the true Hessian inverse. An analogous approach is taken when
updating the Hessian incrementally as in \cref{alg:hessian-update}. In
our experiments, we have set \(\alpha \coloneqq n 10^{-4}\).

\section{Computational Setup Details}\label{sec:computational-setup-details}

The computer used to run had the following specifications:
\begin{description}
  \item[CPU] Intel i7-10510U @ 1.80Ghz (4 cores)
  \item[Memory] 64 GB (3.2 GB/core)
  \item[OS] Fedora 36
  \item[Compiler] GNU GCC compiler v9.3.0, C++17
  \item[BLAS/LAPACK] OpenBLAS v0.3.8
  \item[R version] 4.1.3
\end{description}

\section{Real Data Sets}\label{sec:real-datasets}

All of the data sets except \emph{arcene}, \emph{scheetz}, and \emph{bc\_tcga}
werew retrieved from
\url{https://www.csie.ntu.edu.tw/~cjlin/libsvmtools/datasets/}~\autocite{chang2011,chang2016}.
arcene was retrieved from
\url{https://archive.ics.uci.edu/ml/datasets/Arcene}~\autocite{guyon2004,dua2019}
and scheetz and bc\_tcga from
\url{https://myweb.uiowa.edu/pbreheny}~\autocite{breheny2022}. Their original
sources have been listed in \cref{tab:real-datasets}. In
each case where it is available we use the training partition of the data set
and otherwise the full data set.

\begin{table}[bt]
  \centering
  \caption{Source for the real data sets used in our experiments.\label{tab:real-datasets}}
  \begin{tabular}{ll}
    \toprule
    Dataset            & Sources                                \\
    \midrule
    arcene             & \textcite{guyon2004,dua2019}           \\
    bcTCGA             & \textcite{nationalcancerinstitute2022} \\
    colon-cancer       & \textcite{alon1999}                    \\
    duke-breast-cancer & \textcite{west2001}                    \\
    e2006-log1p        & \textcite{kogan2009}                   \\
    e2006-tfidf        & \textcite{kogan2009}                   \\
    ijcnn1             & \textcite{prokhorov2001}               \\
    madelon            & \textcite{guyon2004}                   \\
    news20             & \textcite{keerthi2005}                 \\
    rcv1               & \textcite{lewis2004}                   \\
    scheetz            & \textcite{scheetz2006}                 \\
    YearPredictionMSD  & \textcite{bertin-mahieux2011,dua2019}  \\
    \bottomrule
  \end{tabular}
\end{table}

\section{Additional Results}\label{sec:additional-results}

In this section, we present additional results related to the
performance of the Hessian Screening Rule.

\subsection{Path Length}\label{sec:path-length-results}

Using the same setup as in \cref{sec:experiments} but with \(n=200\), \(p =
20\,000\) for the high-dimensional setting, we again benchmark the time
required to fit a full regularization path using the different methods studied
in this paper. The results~(\cref{fig:path-length}) show that the
Hessian Screening Method out-performs the studied alternatives except
for the low-dimensional situation and a path length of 10 \(\lambda\)s.
The results demonstrate that our method pays a much smaller price for
increased path resolution compared to the other methods but that the increased
marginal costs of updating the Hessian may make the method less appealing in
this case.

\begin{figure}[tb]
  \centering
  \includegraphics{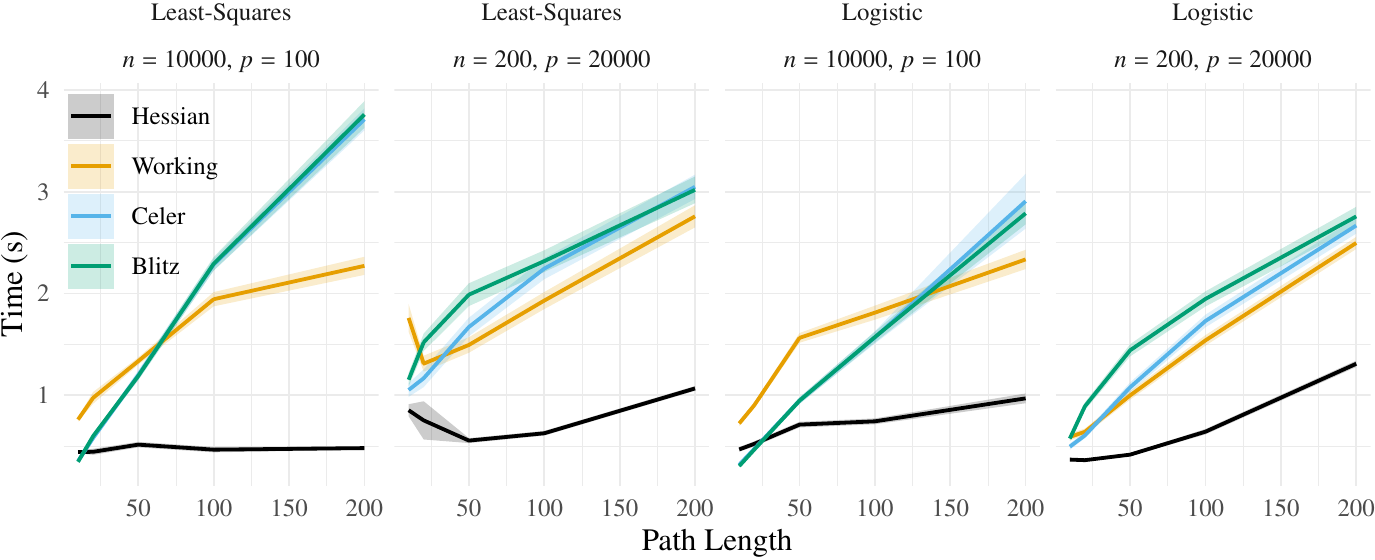}
  \caption{The time in seconds required to fit a full regularization path with length
    given on the x axis.\label{fig:path-length}}
\end{figure}

\subsection{Convergence Tolerance}\label{sec:stopping-threshold-results}

To better understand if and how the stopping threshold used in the solver
affects the performance of the various methods we test, we conduct simulations
where we vary the tolerance, keeping the remaining parameters constant.
We use the same situation as in the high-dimensional scenario (see
\cref{sec:experiments})
but use \(n = 200\), \(p = 20\,000\). We run
the experiment for tolerances \(10^{-3}, 10^{-4}, 10^{-5}\), and \(10^{-6}\).
The results~(\cref{fig:stopping-threshold}) indicate that the choice of
stopping threshold has some importance for convergence time but that the
gap between our method and the alternatives tested never disappears.

\begin{figure}[htb]
  \centering
  \includegraphics{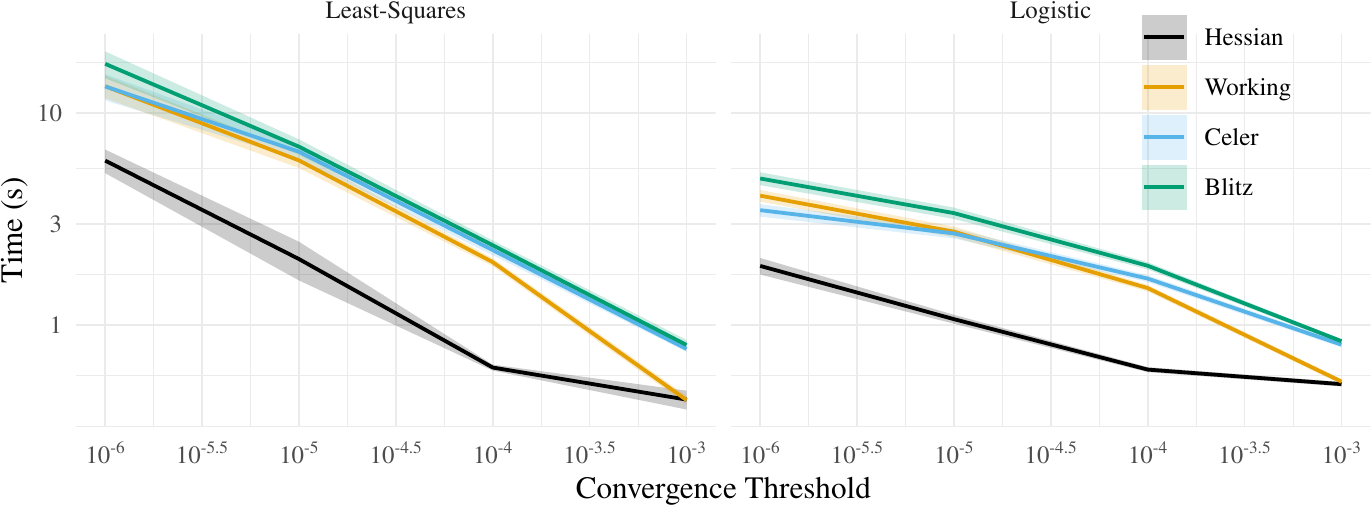}
  \caption{Time required to fit a full regularization path for the
    high-dimensional scenario setup in \cref{sec:experiments} for
    both \(\ell_1\)-regularized least-squares and logistic regression, with \(n =
    200\) and \(p = 20\,000\). Both the x and y axis are on a \(\log_{10}\) scale.
    \label{fig:stopping-threshold}}
\end{figure}

\subsection{The Benefit of Augmenting Heuristic Methods with Gap Safe Screening}
\label{sec:gap-safe-benefit}

To study the effectiveness of augmenting the Hessian Screening and working
methods with a gap-safe check, we conduct experiments using the
high-dimensional setup in \cref{sec:experiments} but with \(n = 200\) and
\(p = 20\,000\), either enabling this augmentation or disabling it. We also
vary the level of correlation, \(\rho\). Each combination is benchmarked
across 20 iterations.

The results indicate that the addition of gap safe screening makes a definite,
albeit modest, contribution to the performance of the methods, particularly
in the case of the working strategy, which is to be expected given that
the working strategy typically runs more KKT checks that the Hessian method
does since it causes many more violations.

\begin{figure}[htbp]
  \centering
  \includegraphics{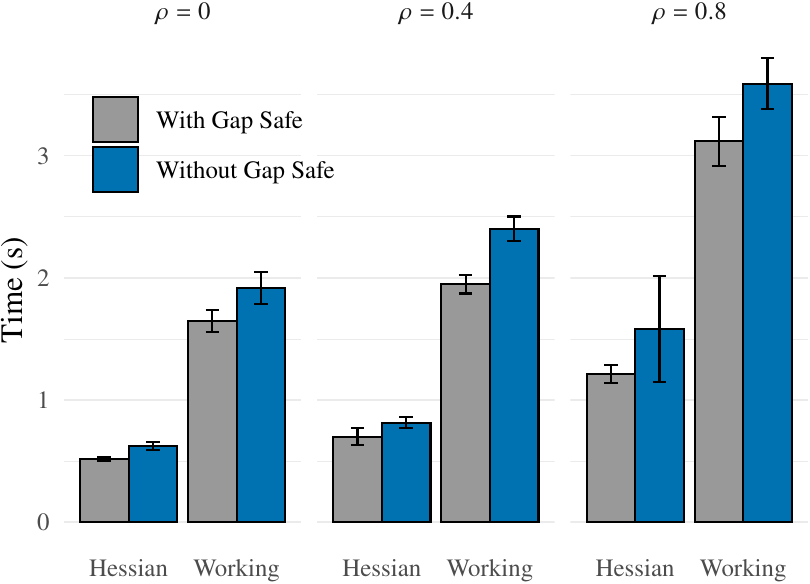}
  \caption{%
    Average time in seconds required to fit a full regularization path for the
    high-dimensional scenario setup in Section~\cref{sec:experiments} for
    \(\ell_1\)-regularized least-squares regression, with \(n = 200\) and \(p =
    20\,000\), using the Hessian and working set methods with or without the
    addition of Gap Safe screening. The bars represent ordinary 95\% confidence
    intervals.\label{fig:gap-safe-benefit}
  }
\end{figure}

\subsection{Effectiveness and Violations}
\label{sec:effectiveness-and-violations}

To study the effectiveness of the screening rule, we conduct as experiment
using the setup in \cref{sec:experiments}~(main paper) but with
\(n=200\) and \(p = 20\,000\). We run 20 iterations and average the number of
screened predictors as well as violations across the entire path.

Looking at the effectiveness of the screening rules, we see that the Hessian
screening rule performs as desired for both \(\ell_1\)-regularized least-squares
and logistic regression~(\cref{fig:efficiency-simulated}), leading to
a screened set that lies very close to the true size. In particular, the
rule works much better than all alternatives in the case of high correlation,

\begin{figure}[hbtp]
  \centering
  \includegraphics{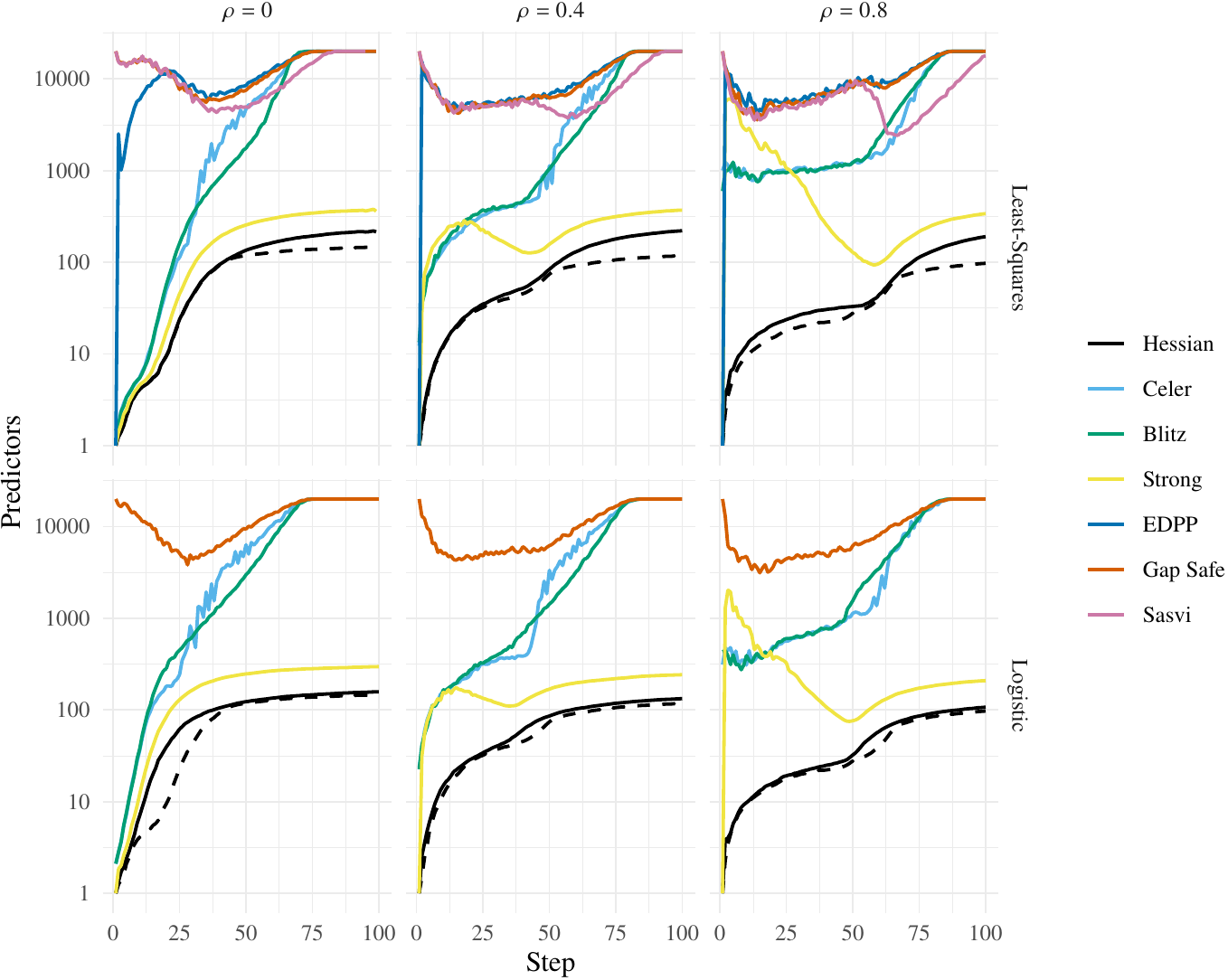}
  \caption{
    The number of predictors screened (included) for each given screening rule,
    as well as the minimum number of active predictors at each step as a dashed
    line. The values are averaged over 20 repetitions in each condition. Note
    that the y-axis is on a \(\log_{10}\) scale.\label{fig:efficiency-simulated}
  }
\end{figure}

In \cref{tab:violations}, we show the average numbers of screened (included)
predictors and violations for the heuristic screening rules across the path. We
note, first, that EDPP never lead to any violations and that the Strong rule
only did so once throughout the experiments. The Hessian rule, on the other
hand, leads to more violations, particularly when there is high correlation. On
the other hand, the Hessian screening rule successfully discards many more
predictors than the other two rules do. And because the Hessian method
always checks for violations in the strong rule set first, which is
demonstrably conservative, these violations are of little importance
in practice.

\begin{table}[hbtp]
  \caption{
    Numbers of screened predictors and violations averaged over
    the entire path and 20 iterations for simulated data with \(n =
    20\,000\) \(p = 200\) and correlation level equal to \(\rho\).\label{tab:violations}}
  \csvreader[
    tabular={
        l
        S[table-format=1.1,round-mode=off,round-precision=1]
        l
        S[table-format=5.0,round-precision=0,round-mode=places]
        S[table-format=1.5,round-precision=2,round-mode=figures]
      },
    before reading=\small\centering,
    table head=\toprule {Model} & {\(\rho\)} & {Method} & {Screened} &
    {Violations}\\\midrule,
    table foot=\bottomrule
  ]%
  {tables/violations.csv}%
  {model=\model, rho=\rhocol, method=\method, screened=\screened,
    violations=\violations}%
  {\model & \rhocol & \method & \screened & \violations}
\end{table}

\subsection{Detailed Results on Real Data}
\label{sec:experiments-realdata-detailed}

In \cref{tab:performance-realdata-details} we show
\cref{tab:performance-realdata} with additional detail, including confidence
intervals and higher figure resolutions. Please see \cref{sec:experiments} for
commentary on these results, where they have been covered in full.

\begin{table}[hbtp]
  \caption{
    Time to fit a full regularization path of \(\ell_1\)-regularized
    least-squares and logistic regression to real data sets. Density and
    time values are rounded to two and four significant figures
    respectively. The estimates are based on 20 repetitions for
    arcene, colon-cancer, duke-breast-cancer, and ijcnn1 and three repetitions
    otherwise. Standard 95\% confidence levels are included. \label{tab:performance-realdata-details}}
  \addtolength{\tabcolsep}{-1.8pt}
  \csvreader[
    tabular={
        l
        S[table-format=6.0,round-mode=off]
        S[table-format=7.0,round-mode=off]
        S[table-format=1.1e-1,scientific-notation=true,round-precision=2]
        l
        l
        S[table-format=4.4]
        S[table-format=4.4]
        S[table-format=4.4]
      },
    before reading=\scriptsize\centering\sisetup{round-mode=figures, round-precision=3},
    table head=\toprule & & & & & & & \multicolumn{2}{c}{95\% CI} \\
    \cmidrule(lr){8-9} Dataset & {\(n\)} & \(p\) & {Density} & {Loss} & {Method} & {Time (s)} & {Lower} & {Upper} \\\midrule,
    table foot=\bottomrule]%
  {tables/realdata-timings-details.csv}%
  {dataset=\dataset, n=\n, p=\p, density=\density, model=\loss,
    screening_type=\method, mean_time=\time, lo=\lo, hi=\hi}%
  {\dataset & \n & \p & \density & \loss & \method & \time & \lo & \hi}%
\end{table}

\subsection{Additional Results on Simulated Data}
\label{sec:experiments-simulateddata-extra}

In \cref{fig:simulateddata-extra-timings}, we show results for the ordinary least-squares
lasso for the Sasvi, Gap Safe, and EDPP methods, which were not included in the main paper.

\begin{figure}[htpb]
  \centering
  \includegraphics[]{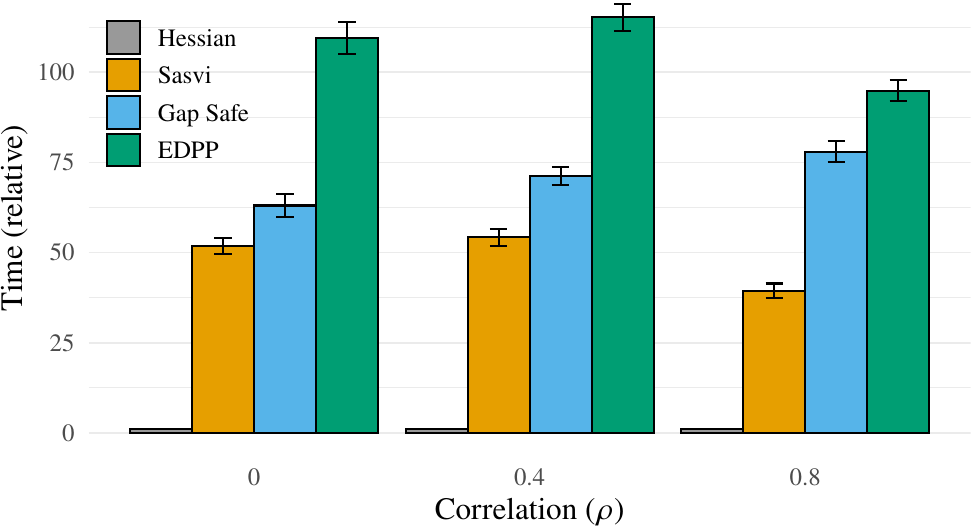}
  \caption{%
    Additional results on simulated data for methods not included in the main article.
    The results correspond to the ordinary (least-squares) lasso with \(n=400\),
    \(p=40\,000\) and varying levels of pairwise correlation between predictors, \(\rho\).
  }
  \label{fig:simulateddata-extra-timings}
\end{figure}

\subsection{Gamma}

In this section we present the results of experiments targeting \(\gamma\),
the parameter for the Hessian rule that controls how much of the
unit bound (used in the Strong Rule) that is included in the
correlation vector estimate from the Hessian rule.

We run 50 iterations of the high-dimensional setup from \cref{sec:experiments}
and measure the number of predictors screened (included) by the Hessian
screening rule, the number of violations, and the time taken to fit the full
path. We vary \(\gamma\) from 0.001 to 0.3.

The results are presented in \cref{fig:gamma}. From the figure it is clear that
the number of violations in fact has a slightly negative impact on the speed at
which the path is fit. We also see that the number of violations is small
considering the dimension of the data set (\(p = 40\,000\)) and approach zero
at \(\gamma\) values around 0.1 for the lowest level of correlation, but have
yet to reach exactly zero at 0.3 for the highest level of correlation. The size
of the screened set increase only marginally as \(\gamma\) increases fro 0.001
to 0.01, but eventually increase rapidly at \(\gamma\) approaches 0.3. Note,
however, that the screened set is still very small relative to the full set of
predictors.

\begin{figure}[htbp]
  \centering
  \includegraphics[]{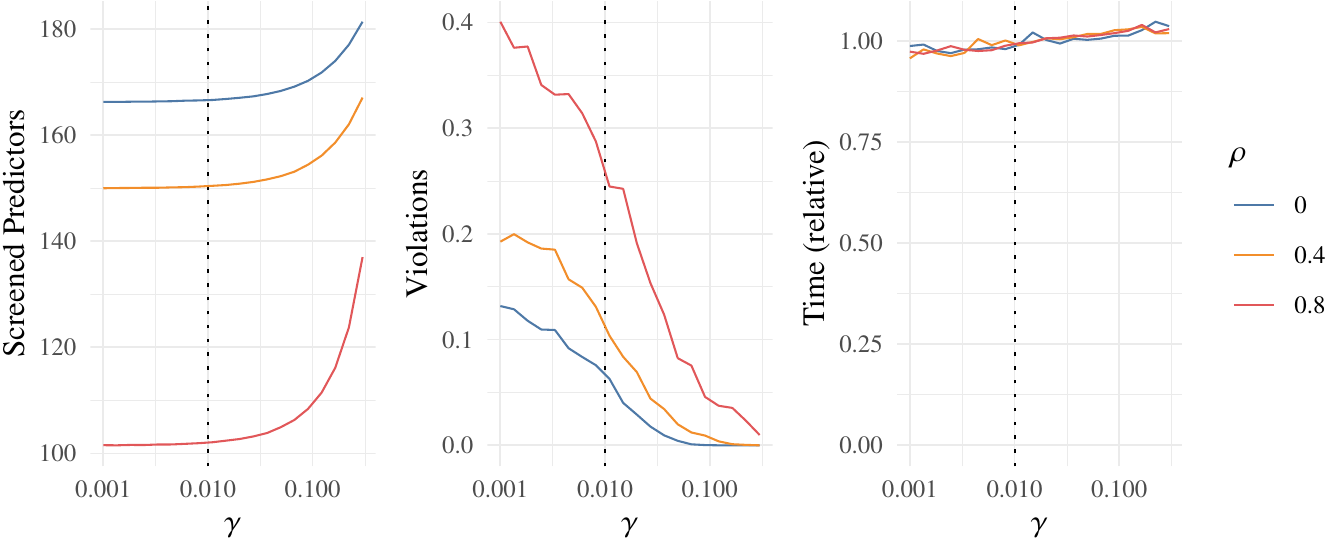}
  \caption{%
    The number of predictors screened (included), the number of
    violations, and the time taken to fit the full path. All measures in the
    plots represent means across combinations of \(\rho\) and \(\gamma\)
    over 50 iterations. The time recorded here is the time relative to the
    mean time for each level of \(\rho\). The choice of \(\gamma\) in this
    work, 0.01, is indicated by a dotted line in the plots. Note that \(x\)
    is on a \(\log_{10}\) scale.\label{fig:gamma}
  }
\end{figure}

\subsection{Ablation Analysis}
\label{sec:ablation}

In this section we report an experiment wherein we study the effects of
the various features of the Hessian screening method by incrementally
adding them and timing the result.

We add features incrementally in the following order, such that
each step includes all of the previous features.
\begin{enumerate}
  \item Hessian screening
  \item Hessian warm starts
  \item Effective updates of the Hessian matrix and its inverse
        using the sweep operator
  \item Gap safe screening
\end{enumerate}
We then run an experiment on a design with \(n = 200\) and \(p = 20\,000\) and
two levels of pairwise correlation between the predictors. The
results~(\cref{fig:ablation}) show that both screening and warm starts make
considerable contributions in this example.

Note that these results are conditional on the order with which they are added
and also on the specific design. The Hessian updates, for instance, make a
larger contribution when \(\min\{n,p\}\) is larger and \(n\) and \(p\) are more
similar. And when \(n \gg p\), the contribution of the warm starts dominate
whereas screening no longer plays as much of a role.

\begin{figure}[htbp]
  \centering
  \includegraphics[]{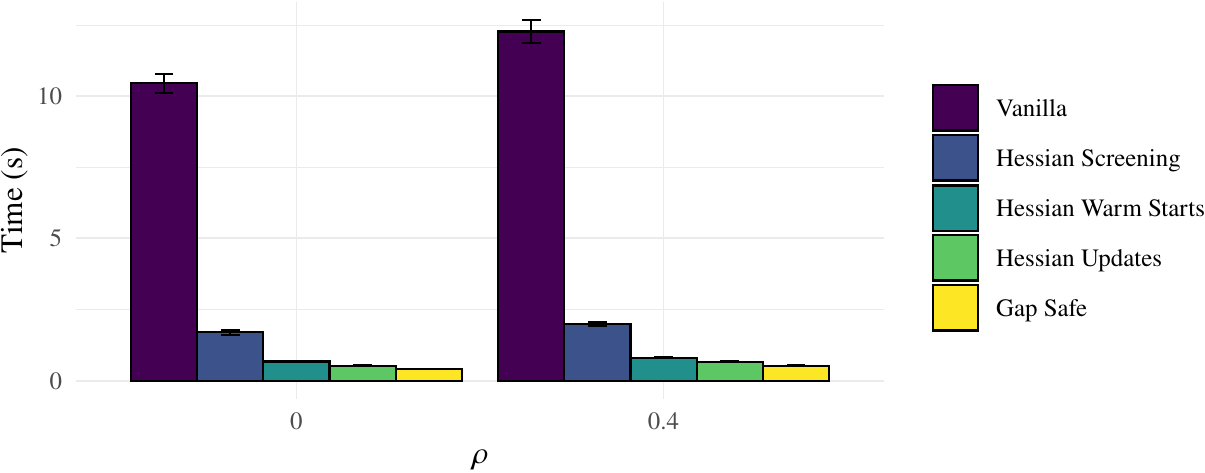}
  \caption{%
    Incremental contribution to the decrease in running time from Hessian
    screening, Hessian warm starts, our effective updates of the Hessian and
    its inverse, and gap safe screening. In other words, \emph{Gap Safe}, for
    instance, includes \emph{all} of the other features, whilst \emph{Hessian Warm
      Starts} includes only \emph{Hessian Screening}. \emph{Vanilla} does not
    include any screening and only uses standard warm starts (from the
    solution at the previous step along the path). The example shows an
    example of ordinary (least-squares) lasso fit to a design with \(n =200\) and
    \(p = 20\,000\) with pairwise correlation between predictors given by \(\rho\).
    (See \cref{sec:experiments} for more details on the setup). The
    error bars indicate standard 95\% confidence intervals. The results are
    based on 10 iterations for each condition.\label{fig:ablation}
  }
\end{figure}

\subsection{\texorpdfstring{\(\ell_1\)}{L1}-Regularized Poisson Regression}
\label{sec:simulateddata-poisson}

In this experiment, we provide preliminary results for \(\ell_1\)-regularized
Poisson regression. The setup is the same as \cref{sec:experiments} except for
the following remarks:
\begin{itemize}
  \item The response, \(y\), is randomly sampled such that \(y_i \sim
        \operatorname{Poisson}\big(\exp(x_i^T\beta)\big)\).
  \item We set \(\zeta\) in the convergence criterion to \(n + \sum_{i=1}^n \log (y_i!) \).
  \item We do not use the line search procedure from Blitz.
  \item Due to convergence issues for higher values of \(\rho\), we use values \(0.0\),
        \(0.15\), and \(0.3\) here. Tackling higher values of \(\rho\) would likely
        need considerable modifications to the coordinate descent solver we use.
    \item The gradient of the negative Poisson log-likelihood is not Lipschitz
        continuous, which means that Gap safe screening~\parencite{ndiaye2017} no
        longer works. As a result, we have excluded the Blitz and Celer
        algorithms, which rely on Gap safe screening, from these benchmarks, and deactivated
        the additional Gap safe screening from our algorithm.
\end{itemize}

The results from the comparison are shown in
\cref{fig:simulateddata-poisson-timings}, showing that our algorithm is
noticeably faster than the working algorithm also in this case.

\begin{figure}[htpb]
  \centering
  \includegraphics[]{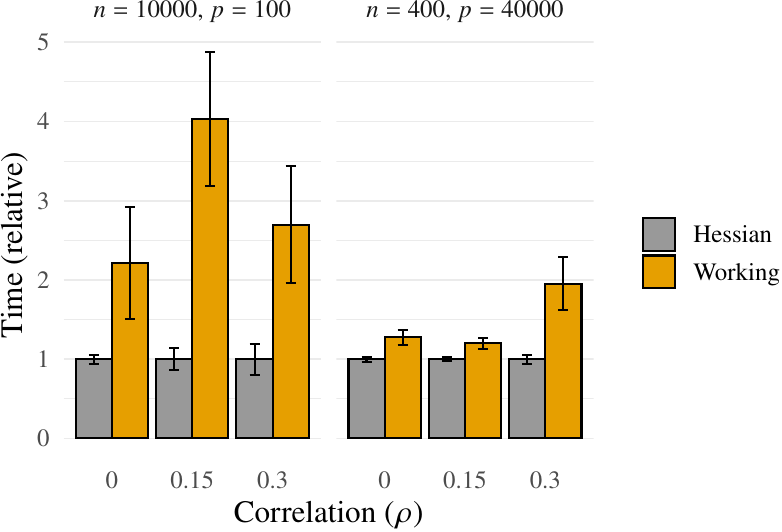}
  \caption{%
    Time to fit a full regularization path for \(\ell_1\)-regularized Poisson
    regression to a design with observations, \(p\) predictors, and pairwise
    correlation between predictors of \(\rho\). Time is relative to the minimal
    mean time in each group. The error bars represent ordinary 95\% confidence
    intervals around the mean.
  }
  \label{fig:simulateddata-poisson-timings}
\end{figure}

\subsection{Runtime Breakdown Along Path}

In this section we take a closer look at the running time of fitting the
full regularization path and study
the impact the Hessian screening rule and its warm starts have on the time
spent on optimization of the problem using coordinate descent (CD).

To illustrate these cases we take a look at three data sets here:
\emph{e2006-tfidf}, \emph{madelon}, and \emph{rcv1}. The first of these,
e2006-tfidf, is a sparse data set of dimensions \(16\,087 \times 150\,360\)
with a numeric response, to which we fit the ordinary lasso. The second two are
both data sets with a binary response, for which we use \(\ell_1\)-regularized
logistic regression. The dimensions of madelon are \(2000 \times 500\) and the
dimensions of rcv1 are \(20\,242 \times 47\,236\).

We study the contribution to the total running time per step, comparing the
Hessian screening rule with the working+ strategy. For the working+ strategy,
all time is spent inside the CD optimizer and in checks
of the KKT conditions. For the Hessian screening rule, time is also spent
updating the Hessian and computing the correlation estimate \(\tilde c^H\).

Beginning with \cref{fig:hesstimefrac-tfidf} we see that the Hessian
strategy dominates the Working+ strategy, which spends most of its
running time on coordinate descent iterations, which the Hessian strategy
ensures are completed in much less time.

\begin{figure}[htbp]
  \centering
  \includegraphics[]{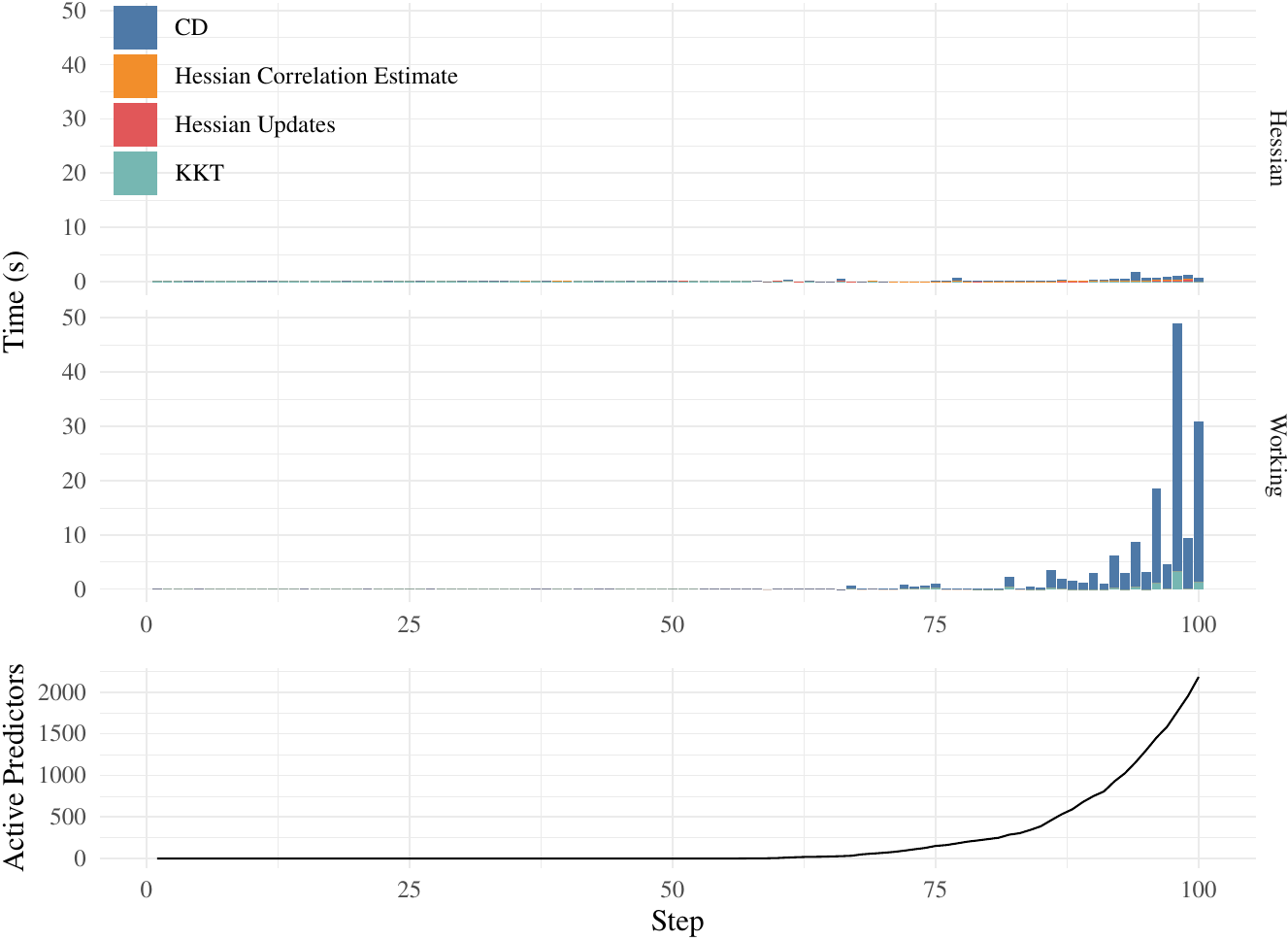}
  \caption{%
    Relative contribution to the full running time when fitting a complete
    regularization path to the \emph{e2006-tfidf} data
    set.\label{fig:hesstimefrac-tfidf}
  }
\end{figure}

In \cref{fig:hesstimefrac-madelon}, we see an example of
\(\ell_1\)-regularized logistic regression. In this case updating the Hessian
exactly (and directly) dominates the other approaches. The size of the problem
makes the cost of updating the Hessian negligible and offers improved screening
and warm starts, which in turn greatly reduces the time spent on coordinate
descent iterations and consequently the full time spent fitting the path.

\begin{figure}[htbp]
  \centering
  \includegraphics[]{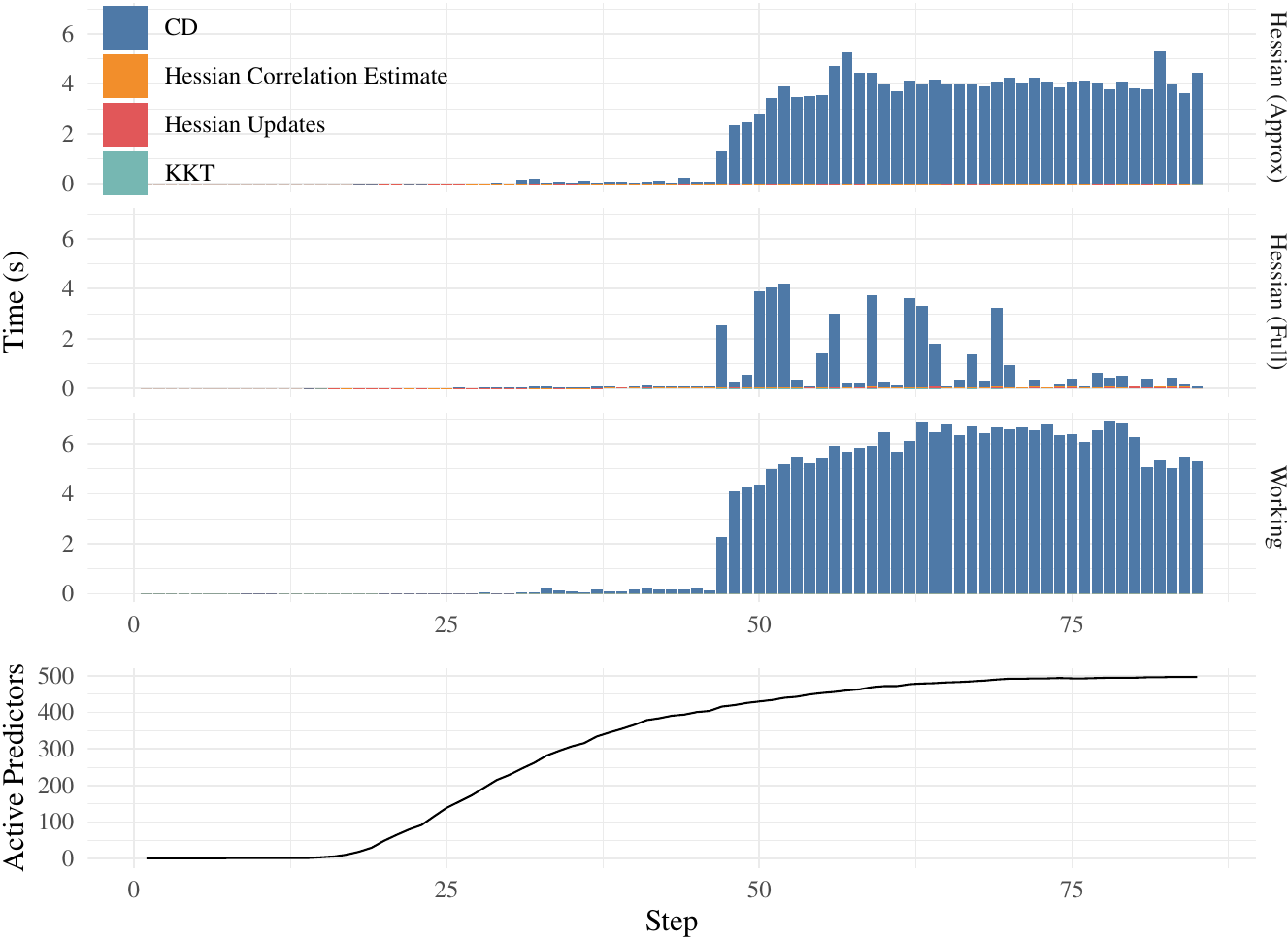}
  \caption{%
    Relative contribution to the full running time when fitting a complete
    regularization path to the \emph{madelon} data set.\label{fig:hesstimefrac-madelon}
  }
\end{figure}

Finally, in \cref{fig:hesstimefrac-rcv1} we consider the \emph{rcv1} data set.
In contrast to the case for \emph{madelon}, the cost of directly forming the
Hessian (and inverse) proves more time-consuming here (although the benefits
still show in the time spent on coordinate descent iterations).

\begin{figure}[htbp]
  \centering
  \includegraphics[]{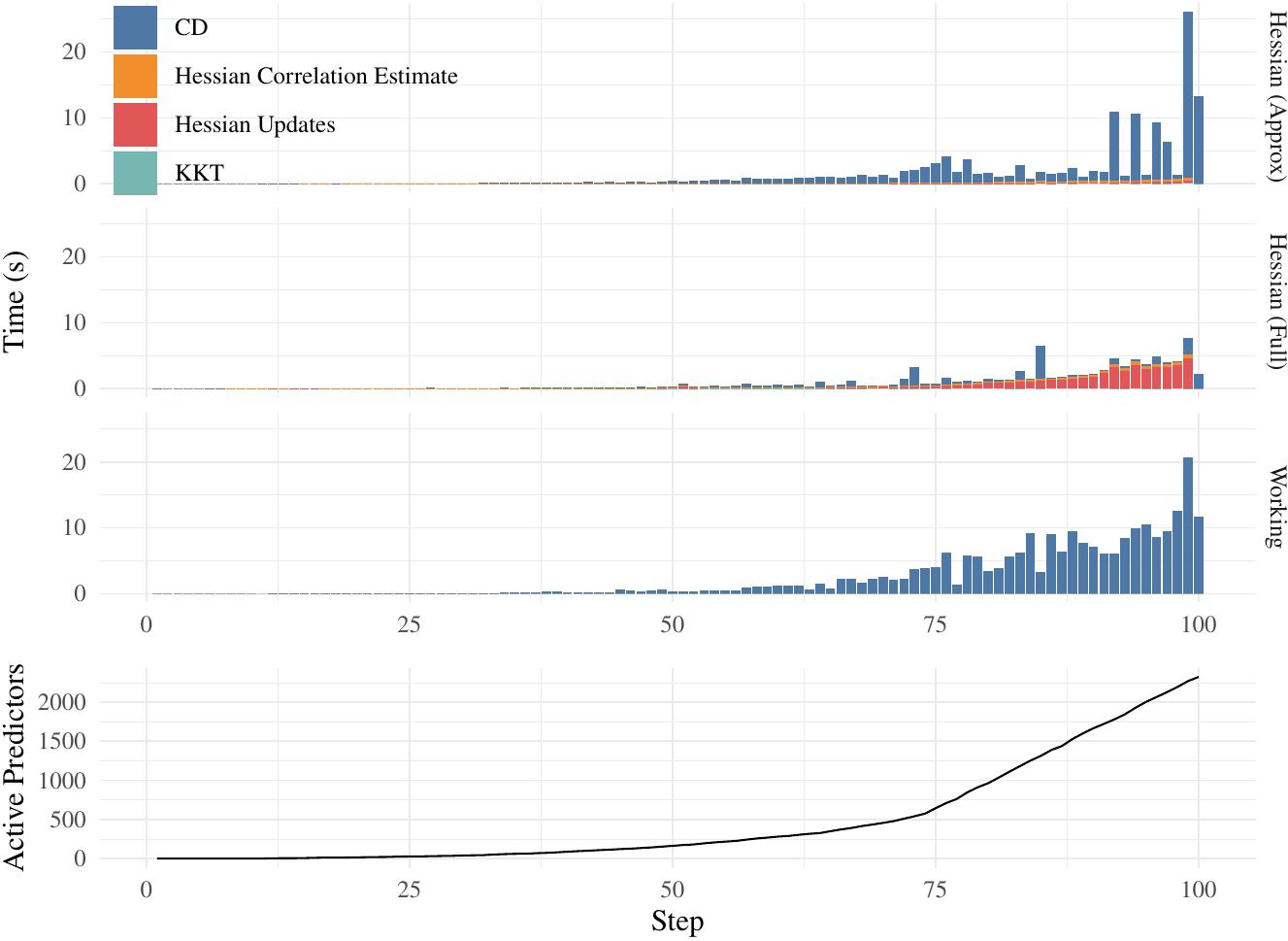}
  \caption{%
    Relative contribution to the full running time when fitting a complete
    regularization path to the \emph{rcv1} data set.
  }\label{fig:hesstimefrac-rcv1}
\end{figure}

As a final remark, note that the pattern by which predictors enter the model
(bottom panels) differ considerably between these three cases
(\cref{fig:hesstimefrac-tfidf,fig:hesstimefrac-madelon,fig:hesstimefrac-rcv1}).
Consider, for instance, \emph{madelon} viz-a-viz \emph{e2006-tfidf}. In
\emph{Approximate Homotopy} (\cref{sec:approximate-homotopy}, main paper),
we discuss a remedy for this solution that is readily available through our
method.

\end{document}